\documentclass[11pt]{article}       
\usepackage{geometry}               
\usepackage{graphicx}
\usepackage{caption}
\usepackage{subcaption}
\usepackage{comment}
\usepackage{amsmath} 
\usepackage{amssymb}  
\usepackage{amsthm}  
\usepackage{bm}
\usepackage{lipsum}
\usepackage[linesnumbered,ruled,vlined]{algorithm2e}
\usepackage{algorithmic}
\usepackage{color}
\usepackage{array}
\usepackage{hyperref}
\usepackage{upgreek}
\usepackage{todonotes}
\usepackage{cite}

\usepackage{multirow}

\newtheorem{proposition}{Proposition}
\newtheorem{defn}{Definition}
\newtheorem{cor}{Corollary}
\newtheorem{lemma}{Lemma}
\newtheorem{theorem}{Theorem}

\newtheorem{assumption}{Assumption}

\numberwithin{equation}{section}

\DeclareMathOperator*{\argmax}{arg\,max}

\def \bx {{\bm{x}}}
\def \bu {{\bm{u}}}

\def \d  {{\rm{d}}}
\def \e  {{\varepsilon}}

\newcommand{\bbr}{\mathbb R}

\title{Hamilton--Jacobi Deep Q-Learning for Deterministic Continuous-Time Systems with Lipschitz Continuous Controls\thanks{This work was supported in part by  the Creative-Pioneering Researchers Program through SNU,  the National Research Foundation of Korea funded by the MSIT(2020R1C1C1009766), the Information and Communications Technology Planning and Evaluation (IITP) grant funded by MSIT(2020-0-00857), and Samsung Electronics.}
} 

\author{Jeongho Kim\thanks{Institute of New Media and Communications,  Seoul National University, Seoul 08826, South Korea, (jhkim206@snu.ac.kr).} 
\and
Jaeuk Shin\thanks{Department of Electrical and Computer Engineering, Automation and Systems Research Institute,  Seoul National University, Seoul 08826, South Korea, (\{sju5379, insoonyang\}@snu.ac.kr).}
\and 
 Insoon Yang\footnotemark[3] 
}

\date{}

\begin{document}
\maketitle
\pagestyle{myheadings}
\thispagestyle{plain}

\begin{abstract}
In this paper, we propose Q-learning algorithms for continuous-time deterministic optimal control problems with Lipschitz continuous controls. Our method is based on a new class of Hamilton--Jacobi--Bellman (HJB) equations  derived from applying the dynamic programming principle to continuous-time Q-functions. 
A novel semi-discrete version of the HJB equation is proposed to design a  Q-learning algorithm that uses data collected in discrete time without discretizing or approximating the system dynamics. We identify the condition under which the Q-function estimated by this algorithm converges to the optimal Q-function. For practical implementation, we propose the \emph{Hamilton--Jacobi DQN}, which extends the idea of deep Q-networks (DQN) to our continuous control setting. This approach does not require actor networks or numerical solutions to optimization problems for greedy actions since the HJB equation provides a simple characterization of optimal controls via ordinary differential equations. We empirically demonstrate the performance of our method through benchmark tasks and high-dimensional linear-quadratic problems. 
\end{abstract}

%
%

\section{Introduction}
\label{intro}

Model-free reinforcement learning (RL) algorithms provide an effective data-driven solution to sequential decision-making problems, in particular, in the discrete-time setting~\cite{Bertsekas1996, Sutton1998, Szepesvari2010}.
Recently, there has been a growing interest in and demand for applying these techniques to complex physical control tasks, motivated by robotic and autonomous systems.  
However, many physical processes evolve in continuous time, requiring the need for RL methods that can systematically handle continuous-time dynamical systems. 
These systems are often described by deterministic ordinary differential equations (ODEs). 
Classical approaches first estimate the model parameters by using system identification techniques and then design a suitable model-based controller~(e.g., \cite{Ljung1998}). 
However, we do not often have such a luxury of having a separate training period for parameter identification, which often requires large-scale high-resolution data. 
Furthermore, when the model parameters change over time, the classical techniques have fundamental limitations in terms of adaptivity. 
The focus of this work is to study a control-theoretic model-free RL method that extends the popular Q-learning~\cite{Watkins1992} and deep Q-networks (DQN)~\cite{Mnih2015} to the continuous-time deterministic optimal control setting.

%

One of the most straightforward ways to tackle such continuous-time control problems is to discretize time, state, and action, and then employ an RL algorithm for discrete Markov decision processes (MDPs).
However, this approach could easily be rendered ineffective  when a fine discretization is used~\cite{Doya2000}.
To avoid the explicit discretization of state and action, several methods have been proposed using function approximators~\cite{Gordon1995}. 
Among those, algorithms that use deep neural networks as  function approximators provide strong empirical evidence for learning high-performance policies, on a range of benchmark tasks~\cite{Todorov2012, Brockman2016, Duan2016, Tassa2018}.
To deal with continuous action spaces,
such discrete-time model-free deep RL methods
numerically solve optimization problems for greedy actions~\cite{Ryu2020} or
use parameterized policies and learn the network parameters via policy gradient~\cite{Schulman2015, Schulman2017}, actor-critic methods~\cite{Lillicrap2016, Mnih2016, Haarnoja2018, Fujimoto2018, Tessler2019}, or normalized advantage functions~\cite{Gu2016}.
However, in these methods it is unclear how to choose the size of discretized time steps  or how the algorithms should be systematically modified  to take into account 
the efficiency and the stability of learning processes according to the characteristics of the continuous-time systems.

The literature regarding continuous-time RL is relatively limited; most of them have tried to avoid explicit discretization using the structural properties of limited classes of system dynamics (for example, see~\cite{Palanisamy2015, Bian2016-2, Vamvoudakis2017, Jiang2015, Kim2020, Bhasin2013, Modares2014,Vamvoudakis2010} for linear or control-affine systems, and see \cite{Bradtke1995} for semi-MDPs with finite state and action spaces). We also refer to \cite{munos2006policy}, where the policy gradient method in continuous-time setting is introduced. However, the reward function does not depend on the control signal in their framework.

In general continuous-time cases, the dynamic programming equation is expressed as a Hamilton--Jacobi--Bellman (HJB) equation that provides a sound theoretical framework. 
Previous methods use HJB equations 
for learning the optimal \emph{state-value function} or its gradient via convergent discretization~\cite{Munos2000}, barycentric interpolation~\cite{munos1999barycentric},
advantage functions~\cite{Dayan1996},
temporal difference algorithms~\cite{Doya2000}, kernel-based approximations~\cite{Ohnishi2018}, adaptive dynamic programming~\cite{yang2017hamiltonian}, path integrals~\cite{theodorou2010generalized,Rajagopal2017} and neural network approximation~\cite{tassa2007least,lutter2020hjb}. 

However, to our knowledge, HJB equations have not been studied for admitting Q-functions as a solution (i.e., state-action value functions) in the previous methods although there have been a few attempts to 
 construct  variants of Q-functions for continuous-time dynamical systems. 
 In \cite{kontoudis2019kinodynamic}, the Q-function for linear time-invariant systems is defined as the sum of the optimal state-value function and the Hamiltonian. 
 Another variant of Q-functions is introduced as the sum of the running cost and the directional derivative of the state-value function \cite{mehta2009q}, which is then approximated by a parameterized family of functions. 
However, in our opinion, the definitions of the Q-function in these works
 are different from the standard state-action value function that is defined as the maximum expected cumulative reward incurred after starting from a particular state with a specific action.
 Moreover, they have only used HJB equations for the state-value function without introducing or using HJB equations for the constructed Q-functions.
The practical performances of these methods have only been demonstrated through low-dimensional tasks. 
More recently, \cite{pmlr-v97-tallec19a} devises a new method combining advantage updating~\cite{baird94} and existing off-policy RL algorithms to propose continuous-time RL algorithms that are robust to time discretization. However, to tackle problems with continuous action spaces, this method uses off-policy actor-critic methods rather than relying  only on the state-value functions.

In this work, we consider continuous-time deterministic optimal control problems with Lipschitz continuous controls in the infinite-horizon discounted setting.
We show that the standard Q-function is well defined in continuous-time under Lipschitz constraints on controls.
Applying the dynamic programming principle to the Q-function, we derive a novel class of HJB equations. 
The HJB equation is shown to admit a unique viscosity solution, which corresponds to the optimal Q-function. 
To the best of our knowledge, this is the first attempt to rigorously characterize the HJB equations for Q-functions in continuous-time control. 
The HJB equations   provide a simple model-free characterization of optimal controls via ODEs and
a theoretical basis for our Q-learning method. 
We propose a new semi-discrete version of the HJB equation
to obtain a Q-learning algorithm that uses sample data collected in discrete time 
without discretizing or approximating the continuous-time dynamics. 
By design, it attains the flexibility to choose the sampling interval to take into account the features of continuous-time systems, but without the need for sophisticated ODE discretization methods.
We provide a convergence analysis that suggests a limit for the sampling interval for the convergence guarantee.
This study may open a new exciting avenue of research that connects HJB equations and Q-learning domain.

For a practical implementation of our HJB-based Q-learning, 
we combine it with the idea of DQN.
This new model-free off-policy deep RL algorithm, which we call the \emph{Hamilton-Jacobi DQN} (HJ DQN), is as simple as DQN but capable of solving continuous-time problems without discretizing the system dynamics or the action space.
Instead of using any parameterized policy or numerically optimizing the estimated Q-functions to compute greedy actions, 
HJ DQN benefits from the simple ODE characterization of optimal controls, which are obtained in our theoretical analysis of the HJB equations. 
Thus, our algorithm is computationally light and easy to implement, thereby requiring less hyperparameter tuning compared to actor-critic methods for continuous control. 
We evaluate our algorithm on OpenAI benchmark tasks and high-dimensional  linear-quadratic (LQ) control problems. The result of our experiments suggests that actor networks in actor-critic methods may be replaced by the optimal control obtained via our HJB equation.

This paper is significantly expanded from a preliminary conference version~\cite{kim2020hamilton}. 
A Q-learning algorithm and its DQN variant are newly designed in a principled manner to use transition data collected in discrete time.
Furthermore, convergence properties of our Q-learning method are carefully studied in this paper. 
It contains the results of more thorough numerical experiments
on several benchmark tasks and LQ problems, as well as ablation studies.

The remainder of this paper is organized as follows. In Section \ref{sec:hjb}, we define the Q-functions for continuous-time optimal control problems with Lipschitz continuous controls and derive the associated HJB equations. 
We also characterize optimal control dynamics via an ODE.
 In Section \ref{sec:semi}, we propose a Q-learning algorithm based on the semi-discrete HJB equation and identify its convergence properties. 
In Section \ref{sec:hjdqn}, we introduce the HJ DQN algorithm and discuss its features. 
Section \ref{sec:exp} provides the results of our experiments on benchmark problems  as well as LQ control problems.
All the mathematical proofs are contained in Appendix~\ref{app:pf}. 


\section{Hamilton--Jacobi--Bellman Equations for Q-Functions}\label{sec:hjb}

Consider a  continuous-time dynamical system of the form\footnote{Here, $\dot{x}$ denotes ${\mathrm{d}x}/{\mathrm{d}t}$.}
\begin{equation}\label{sys}
\dot{x} (t) = f(x(t), a(t)), \quad t > 0,
\end{equation}
where $x(t) \in \mathbb{R}^n$ and $a(t) \in \mathbb{R}^m$ are the system state and the control action, respectively. 
Here, the vector field $f: \mathbb{R}^n \times \mathbb{R}^m \to \mathbb{R}^n$ is an unknown function.
The standard infinite-horizon discounted optimal control problem can be formulated as\footnote{Although the focus of this work is deterministic control,
	one may also consider its stochastic counterpart. We  briefly discuss the extension of our method to the stochastic control setting in Appendix \ref{app:stoch}.}
\begin{equation}\label{opt}
\sup_{a \in \mathcal{A}} J_{\bm{x}} (a) := 
\int_0^\infty e^{-\gamma t} r (x(t), a(t)) \: \mathrm{d} t,
\end{equation}
with  $x(0) = \bm{x}$, where $r: \mathbb{R}^n \times \mathbb{R}^m \to \mathbb{R}$ is an unknown reward function of interest and $\gamma > 0$ is a discount factor.
We follow the convention in continuous-time deterministic optimal control that considers control trajectory, instead of control policy, as the optimization variable~\cite{Bardi1997}.

The (continuous-time) Q-function of \eqref{opt} is defined as
\begin{equation}\label{Q-function}
Q(\bm{x}, \bm{a}) := \sup_{a \in \mathcal{A}} \bigg \{
\int_0^\infty e^{-\gamma t} r (x(t), a(t)) \: \mathrm{d} t  \mid x(0) = \bm{x}, a(0) = \bm{a}
\bigg \},
\end{equation}
which represents the maximal reward incurred from time $0$ when starting from $x(0) = \bm{x}$ with $a(0) = \bm{a}$.
Suppose for a moment that the set of admissible controls $\mathcal{A}$ has no particular constraints, i.e., $\mathcal{A}:= \big \{
a: \mathbb{R}_{\geq 0} \to \mathbb{R}^m \mid a \mbox{ measurable}
\big \}$. Then,  $Q(\bm{x}, \bm{a})$ reduces to the standard optimal value function $v(\bm{x}) := \sup_{a \in \mathcal{A}} \big \{
\int_0^\infty e^{-\gamma t} r (x(t), a(t)) \: \mathrm{d} t \mid x(0) = \bm{x}
\big \}$ for all $\bm{a} \in \mathbb{R}^m$ since the action can be  switched immediately  from $\bm{a}$ to an optimal control and 
in this case $\bm{a}$ does not affect the total cost or the system trajectory in the continuous-time setting.

\begin{proposition}\label{prop:lip}
	Suppose that $\mathcal{A}:= \big \{
	a: \mathbb{R}_{\geq 0} \to \mathbb{R}^m \mid a \mbox{ measurable}
	\big \}$. Then, the optimal Q-function \eqref{Q-function} corresponds to the optimal value function $v$ for each $\bm{a} \in \bbr^m$, i.e., $Q(\bx,\bm{a},t)=v(\bx,t)$ for all $(\bx, \bm{a}, t) \in  \bbr^n \times \bbr^m \times [0,T]$.
\end{proposition}

Thus, if $\mathcal{A}$ is chosen as above, 
the Q-function has no additional interesting property under the standard choice of $\mathcal{A}$.\footnote{This observation is consistent with the previously reported result on the continuous limit of $Q$-functions~\cite{baird94, pmlr-v97-tallec19a}.}
Motivated by the observation,  we restrict the control $a(t)$ to be a Lipschitz continuous function in $t$. 
Since any Lipschitz continuous function is differentiable almost everywhere, 
we choose the set of admissible controls as 
\[
\mathcal{A}:= \big \{
a: \mathbb{R}_{\geq 0} \to \mathbb{R}^m \mid a \mbox{ measurable}, \: |\dot{a}(t)|\le L \mbox{ a.e.}
\big \},
\]
where $| \cdot |$ denotes the standard Euclidean norm, and $L$ is a fixed constant.
From now on, we will focus on the optimal control problem~\eqref{opt} with Lipschitz continuous controls, i.e., $|\dot{a}(t)|\le L \mbox{ a.e.}$, and the corresponding Q-function~\eqref{Q-function}.


Our first step is to study the structural properties of the optimality equation and the optimal control via dynamic programming.
Using the discovered structural properties,
a DQN-like algorithm is then designed to solve the optimal control problem~\eqref{opt} in a model-free manner.

\subsection{Dynamic Programming and HJB Equations}

By the dynamic programming principle, we have
\begin{equation}\nonumber
Q(\bm{x}, \bm{a}) = \sup_{a \in \mathcal{A}} 
\bigg \{
\int_t^{t+h} e^{-\gamma (s-t)} r(x(s), a(s)) \: \mathrm{d} s + 
e^{-\gamma h} Q(x(t+h), a(t+h)) \mid x(t) = \bm{x}, a(t) = \bm{a}
\bigg \}
\end{equation}
for any $h > 0$. 
Rearranging this equality, we obtain
\begin{equation}\nonumber
\begin{split}
0&= \sup_{a \in \mathcal{A}} 
\bigg \{
\frac{1}{h} \int_t^{t+h} e^{-\gamma (s-t)} r(x(s), a(s)) \: \mathrm{d} s+ \frac{1}{h} [Q(x(t+h), a(t+h)) - Q(\bm{x}, \bm{a}) ]\\
&\hspace{4cm} + \frac{e^{-\gamma h} - 1}{h} Q(x(t+h), a(t+h))
\mid x(t) = \bm{x}, a(t) = \bm{a}
\bigg \}.
\end{split}
\end{equation}
Letting $h$ tend to zero and
assuming for a moment that the Q-function is continuously differentiable, its Taylor expansion yields
\begin{equation}\nonumber
\gamma Q (\bm{x}, \bm{a}) - \nabla_{\bm{x}} Q \cdot f(\bm{x}, \bm{a}) - \sup_{\bm{b} \in \mathbb{R}^m, | \bm{b} | \leq L} \nabla_{\bm{a}} Q \cdot \bm{b}- r(\bm{x}, \bm{a}) = 0,
\end{equation}
where the optimization variable $\bm{b}$ represents $\dot{a} (t)$.
Note that the supremum is attained at 
$\bm{b}^\star = L \frac{\nabla_{\bm{a}} Q}{|\nabla_{\bm{a}} Q|}$.
Thus, we obtain 
\begin{equation}\label{HJB}
\gamma Q (\bm{x}, \bm{a}) - \nabla_{\bm{x}} Q \cdot f(\bm{x}, \bm{a}) - L {|\nabla_{\bm{a}} Q|} - r(\bm{x}, \bm{a}) = 0,
\end{equation}
which is the \emph{HJB equation for the Q-function}.
However, the Q-function is not continuously differentiable in general. This motivates us to consider a weak solution of the HJB equation. Among several types of weak solutions, it is shown in Appendix~\ref{app:vis} that the Q-function corresponds to the unique \emph{viscosity solution}~\cite{Crandall1983} of the HJB equation under the following assumption:
\begin{assumption}\label{ass:bl}
	The functions $f$ and $r$ are bounded and Lipschitz continuous, i.e., there exists a constant $C$ such that $\|f \|_{L^\infty} + \|r \|_{L^\infty} < C$ and $\|f \|_{\mathrm{Lip}} + \|r \|_{\mathrm{Lip}} < C$, where $\| \cdot \|_{\mathrm{Lip}}$ denotes a Lipschitz constant of argument. 
\end{assumption}

\subsection{Optimal Controls}

In the derivation of the HJB equation above, we deduce that an optimal control $a$ must satisfies $\dot{a} = L \frac{\nabla_{\bm{a}} Q}{|\nabla_{\bm{a}} Q|}$ when $Q$ is differentiable.
The viscosity solution framework~\cite{Bardi1997}  can be used to obtain the following  more rigorous characterization of optimal controls when the Q-function is not differentiable.

\begin{theorem}\label{optimal}
	Suppose that Assumption~\ref{ass:bl} holds. Consider a control trajectory $a^\star (s)$, $s\geq t$, defined by
	\begin{equation}\label{opt_con}
	\dot{a}^\star (s) = L \frac{p_2}{| p_2|} \quad \forall p = (p_1, p_2) \in D^\pm Q (x^\star (s), a^\star (s))  
	\end{equation}
	for a.e. $s \geq t$, and $a^\star (t) = \bm{a}$, where $\dot{x}^\star = f(x^\star, a^\star)$ for $s\geq t$ and $x^\star (t) = \bm{x}$. Assume that the function $Q$ is locally Lipschitz in a neighborhood of $(x^\star (s), a^\star (s))$ and that $D^+ Q(x^\star (s), a^\star (s)) = \partial Q (x^\star(s), a^\star (s))$ for a.e. $s \geq t$.\footnote{Here, $D^+ Q$ and $D^- Q$ denote the super- and sub-differentials of $Q$, respectively, and $D^\pm Q := D^+ Q \cup D^- Q$. At a point $(\bm{x}, \bm{a})$ where $Q$ is differentiable, the super- and sub-differentials are identical to the singleton of the classical derivative of $Q$. Moreover,  $\partial Q$ denotes the Clarke's generalized gradient of $Q$ (see, e.g., p. 63 of \cite{Bardi1997}). Note that the right-hand side of ODE~\eqref{opt_con} can be arbitrarily chosen when $p_2 = 0$.}  Then, $a^\star$ is optimal among those in $\mathcal{A}$ such that $a(t) = \bm{a}$, i.e., it satisfies
	
	\begin{equation}\label{optimality}  
	a^\star \in \argmax_{a \in \mathcal{A}} \bigg\{ \int_t^\infty e^{-\gamma (s - t)} r (x(s), a(s)) \: \mathrm{d}s,\Big\vert ~ x(t) = \bm{x}, a(t) = \bm{a}
	\bigg \}.
	\end{equation}
	If, in addition, 
	\[
	\bm{a} \in \argmax_{\bm{a}' \in \mathbb{R}^m} Q(\bm{x}, \bm{a}'),
	\]
	then $a^\star$ is an optimal control, i.e., it satisfies
	\begin{equation*} \nonumber
	a^\star \in \argmax_{a \in \mathcal{A}} \left\{ \int_t^\infty e^{-\gamma (s - t)} r (x(s), a(s)) \: \mathrm{d}s ~\Big\vert ~ x(t) = \bm{x} \right\}.
	\end{equation*}
\end{theorem}

Note that at a point $(\bm{x}, \bm{a})$ where $Q$ is differentiable, the ODE~\eqref{opt_con} is  simplified to $\dot{a}^\star = L \frac{\nabla_{\bm{a}}Q (x^\star, a^\star)}{|\nabla_{\bm{a}}Q (x^\star, a^\star)|}$. A useful implication of this theorem is that for any $\bm{a} \in \mathbb{R}^m$, an optimal control in $\mathcal{A}$ such that $a(t) = \bm{a}$ can  be obtained using the ODE \eqref{opt_con} with the initial condition $a^\star (t) = \bm{a}$. Thus, when the control is initialized as an arbitrary value $\bm{a}$ at arbitrary time $t$ in Q-learning, we can still use the ODE~\eqref{opt_con} to obtain an optimal control. Another important implication of Theorem~\ref{optimal} is that an optimal control can be constructed without numerically solving any optimization problem. This salient feature assists in the design of a computationally efficient DQN algorithm for continuous control without involving any explicit optimization nor any actor network.

\section{Hamilton--Jacobi Q-Learning}\label{sec:semi}

\subsection{Semi-Discrete HJB Equations and Asymptotic Consistency}

In practice, even though the underlying physical process evolves in continuous time, the observed data, such as sensor measurements, are collected in discrete (sample) time. 
To design a concrete algorithm for learning the Q-function using such discrete-time data, we propose a novel semi-discrete version of the HJB equation~\eqref{HJB} \emph{without discretizing or approximating the continuous-time system}.
Let $h > 0$ be a fixed \emph{sampling interval}, and  let
$\mathcal{B}:= \{ b := \{ b_k\}_{k=0}^\infty \mid b_k \in \mathbb{R}^m, | b_k | \leq L \}$,
where $b_k$ is analogous to $\dot{a}(t)$ in the continuous-time case.
Given $(\bm{x}, \bm{a}) \in \mathbb{R}^n \times \mathbb{R}^m$ and 
a sequence $b \in \mathcal{B}$, we let
\[
Q^{h,b} (\bm{x}, \bm{a}) := h \sum_{k=0}^\infty r(x_k, a_k) (1-\gamma h)^k,
\]
where  $\{ (x_k, a_k) \}_{k=0}^\infty$ is defined by
$x_{k+1} = \xi(x_k, a_k; h)$ and  
$a_{k+1} = a_k + h b_k$ with $(x_0, a_0) = (\bm{x}, \bm{a})$. Here, $\xi(x_k,a_k;h)$ denotes the state  of \eqref{sys} at time $t=h$ with initial state $x(0)=x_k$ and constant action $a(t)\equiv a_k$, $t \in [0, h)$. It is worth emphasizing that our semi-discrete approximation does \emph{not}  approximate the system dynamics 
and thus is more accurate than the standard semi-discrete method.
The optimal semi-discrete Q-function $Q^{h,\star}: \mathbb{R}^n \times \mathbb{R}^m \to \mathbb{R}$ is then defined by
\begin{equation}\label{Q_semi}
Q^{h,\star} (\bm{x}, \bm{a}) := \sup_{b \in \mathcal{B}} Q^{h,b} (\bm{x}, \bm{a}).
\end{equation}
Then, $Q^{h,\star}$ satisfies a semi-discrete version of the HJB equation~\eqref{HJB}.
\begin{proposition}\label{prop:semi}
	Suppose that $0 < h < \frac{1}{\gamma}$. Then, the function $Q^{h,\star}$ is a solution to the following semi-discrete HJB equation:
	\begin{equation}\label{semiHJB}
	Q^{h,\star} (\bm{x}, \bm{a})= h r(\bm{x}, \bm{a}) +(1-\gamma h) \sup_{|\bm{b}| \leq L}
	Q^{h,\star} ( \xi(\bm{x}, \bm{a}; h), \bm{a} + h \bm{b}). 
	\end{equation}
\end{proposition}

Under Assumption~\ref{ass:bl}, $Q^{h,\star}$ coincides with the unique solution of the semi-discrete HJB equation~\eqref{semiHJB}. 
Moreover, the optimal semi-discrete Q-function converges uniformly to its original counterpart in every compact subset of $\mathbb{R}^n \times \mathbb{R}^m$. 

\begin{proposition}\label{conv1}
	Suppose that $0 < h < \frac{1}{\gamma}$ and that Assumption~\ref{ass:bl} holds.
	Then, the function $Q^{h,\star}$ is the unique solution to the semi-discrete HJB equation~\eqref{semiHJB}. Furthermore, we have
	\[
	\lim_{h\to 0} \sup_{(\bm{x}, \bm{a}) \in K,  K \mathrm{\tiny compact}}  
	| Q^{h,\star} (\bm{x}, \bm{a}) - Q (\bm{x}, \bm{a}) |  = 0.
	\]
\end{proposition}
%
This proposition justifies the use of the semi-discrete HJB equation  for  small $h$.
We aim to estimate the optimal Q-function using sample data collected in discrete time, enjoying the benefits of both the semi-discrete HJB equation~\eqref{semiHJB} and the original HJB equation~\eqref{HJB}.
Namely, the semi-discrete version yields to naturally make use of  Q-learning and DQN, and the original version provides an optimal control via \eqref{opt_con} without requiring a numerical solution for any optimization problems or actor networks as we will see in Section~\ref{sec:hjdqn}.


%
%
%

\subsection{Convergence Properties}

Consider the following model-free update of Q-function using the semi-discrete HJB equation~\eqref{semiHJB}:
In the $k$th iteration, for each $(\bm{x}, \bm{a})$
we collect data $(x_k := \bm{x}, a_k := \bm{a}, r_k, x_{k+1})$ and update the Q-function, with learning rate $\alpha_k$, by
\begin{equation} \label{Q_sync}
{Q}_{k+1}^h (\bm{x}, \bm{a}) :=  (1-\alpha_k) {Q}_k^h (\bm{x}, \bm{a})
\; =\alpha_k \Big [ h r_k  + (1-\gamma h) \sup_{|\bm{b} | \leq L} {Q}_k^h (x_{k+1},  \bm{a} + h \bm{b}) 
\Big ],
\end{equation}
where $x_{k+1}$ is obtained by running (or simulating) the continuous-time system from $x_k$ with action $a_k$ fixed for $h$ period without any approximation, i.e., $x_{k+1} = \xi(x_k, a_k; h)$, and $r_k = r (x_k, a_k)$. 
We refer to this synchronous Q-learning as \emph{Hamilton--Jacobi Q-learning}.
Note that this method is not practically useful because 
the update must be performed for all state-action pairs in the continuous space. 
In the following section, we propose a DQN-like algorithm to approximately perform HJ Q-learning employing deep neural networks as function approximators. 
Before doing so, we identify conditions under which  the Q-function updated by \eqref{Q_sync} converges to the optimal semi-discrete Q-function~\eqref{Q_semi}  in $L^\infty$.

\begin{theorem}\label{thm:conv2}
	Suppose that $0 < h < \frac{1}{\gamma}$, $0\le \alpha_k\le 1$ and that Assumption~\ref{ass:bl} holds.
	If the sequence $\{ \alpha_k\}_{k=0}^\infty$ of learning rates satisfies
	$\sum_{k=0}^\infty \alpha_k = \infty$,
	then 
	\[
	\lim_{k\to \infty} \| {Q}_k^h - Q^{h,\star} \|_{L^\infty} = 0.
	\]
\end{theorem}

Finally, by Propositions~\ref{conv1} and Theorem~\ref{thm:conv2}, we establish the following convergence result associating HJ Q-learning~\eqref{Q_sync} and the optimal Q-function in the original continuous-time setting.

\begin{cor}\label{cor:conv}
	Suppose that $0\le \alpha_k\le 1$ and that Assumption 1 holds. If the sequence $\{\alpha_k\}_{k=0}^\infty$ of learning rates satisfies $\sum_{k=0}^\infty \alpha_k=\infty$
	then, for each $0<h<\frac{1}{\gamma}$, there exists $k_h$ such that
	$h \sum_{\tau=0}^{k_h-1} \alpha_\tau \to \infty$ as $h\to0$.
	Moreover, for such a choice of $k_h$, we have
	\[\lim_{h\to0}\sup_{k\ge k_h}\sup_{(\bm{x}, \bm{a}) \in K,  K \mathrm{\tiny compact}}  
	|Q^h_{k}(\bx,\bm{a})-Q (\bx,\bm{a})|=0.\]
\end{cor}

\section{Hamilton--Jacobi DQN}\label{sec:hjdqn}

The convergence result in the previous section suggests that the optimal Q-function can be estimated in a model-free manner through the use of the semi-discrete HJB equation. 
However, as mentioned, it is intractable to directly implement HJ Q-learning~\eqref{Q_sync} over a continuous state-action space.
As a practical function approximator, we employ deep neural networks. 
We then propose the \emph{Hamilton--Jacobi DQN} that approximately performs the update~\eqref{Q_sync} \emph{without discretizing or approximating the continuous-time system}. 
Since our algorithm has no actor, we only consider a parameterized Q-function $Q_\theta (\bx, \bm{a})$, where $\theta$ is the parameter vector of the network.  


As with DQN, we use a separate target function $Q_{\theta^-}$, where the network parameter vector $\theta^-$ is updated more slowly than $\theta$. 
This allows us to update $\theta$ by solving a regression problem with an almost fixed target, resulting in consistent and stable learning~\cite{Mnih2015}.
We also use experience replay by storing transition data $(x_k, a_k, r_k, x_{k+1})$ in a buffer with fixed capacity and by randomly sampling a mini-batch of transition data $\{(x_j, a_j, r_j, x_{j+1})\}$ to update the target value. 
This reduces bias by
breaking the correlation between sample data that are sequential states~\cite{Mnih2015}. 

When setting the target value in DQN, the target Q-function needs to be maximized over all admissible actions, i.e.,
$y_j^- := h r_j + \gamma'  \max_{\bm{a}} Q_{\theta^-} (x_{j+1}, \bm{a})$. 
Evaluating the maximum is tractable in the case of discrete action spaces. 
However, in our case of continuous action spaces, it is computationally challenging to maximize the target Q-function with respect to the action variable. 
To resolve this issue, we go back to the original HJB equation and use the corresponding optimal action in Theorem~\ref{optimal}.
Specifically, we consider the action dynamics~\eqref{opt_con} with
$b_j := L \frac{\nabla_{\bm{a}} Q_{\theta^-}(x_{j}, a_j)}{|\nabla_{\bm{a}} Q_{\theta^-}(x_{j}, a_j)|}$ fixed over sampling interval $h$ to obtain
\begin{equation}\label{app_opt}
a_{j+1} = a_j + hb_j := a_j + h L \frac{\nabla_{\bm{a}} Q_{\theta^-} (x_{j}, a_j)}{|\nabla_{\bm{a}} Q_{\theta^-} (x_{j}, a_j)|}.
\end{equation}
Using this optimal control action, we can approximate the maximal target Q-function value as
$\max_{|\bm{a} - a_j| \leq hL} Q_{\theta^-} (x_{j+1}, \bm{a}) \approx 
Q_{\theta^-}  (x_{j+1}, a_j + h b_j)$.
This approximation becomes more accurate as $h$ decreases.

\begin{proposition}\label{prop:diff}
	Suppose that $Q_{\theta^-}$ is twice continuously differentiable with bounded first and second  derivatives. 
	If $\nabla_{\bm{a}} Q_{\theta^-} (x_j, a_j) \neq 0$, we have
	\[
	\lim_{h\to 0} \Big |
	\max_{|\bm{a} - a_j| \leq hL} Q_{\theta^-} (x_{j+1}, \bm{a}) - Q_{\theta^-}  (x_{j+1}, a_j + h  b_j)
	\Big | = 0.
	\]
	Moreover, the  difference above is $O(h^2)$ as $h\to 0$. 
\end{proposition}

The major advantage of using the optimal action obtained in the continuous-time case is to avoid explicitly solving the nonlinear optimization problem $\max_{|\bm{a} - a_j| \leq hL} Q_{\theta^-} (x_{j+1}, \bm{a})$, which is computationally demanding. 
With this choice of  target Q-function value and the semi-discrete HJB equation~\eqref{semiHJB}, we set the target value as
$y_j^- := h r_j + (1-\gamma h) Q_{\theta^-}  (x_{j+1}, a_j + h  b_j )$.
To mitigate the overestimation of Q-functions, we can employ double Q-learning~\cite{Hasselt2016} by simply modifying $b_j$ as $b_j := L \frac{\nabla_{\bm{a}} Q_{\theta}(x_{j}, a_j)}{|\nabla_{\bm{a}} Q_{\theta}(x_{j}, a_j)|}$ to use a greedy action with respect to $Q_\theta$ instead of $Q_{\theta^-}$.
In this double Q-learning version, Proposition~\ref{prop:diff} remains valid except for the $O(h^2)$ convergence rate.
The network parameter $\theta$ can then be trained to minimize the  loss function $\sum_j (y_j^- - Q_\theta (x_j, a_j))^2$.
For exploration, we add the additional Gaussian noise $\varepsilon\sim N(0,\sigma^2I_m)$ to generate the next action as $a_{k+1}:= a_k +hL\frac{\nabla_{\bm a} Q_\theta(x_k,a_k)}{|\nabla_{\bm a} Q_\theta(x_k,a_k)|}+\varepsilon$. The overall algorithm is presented in Algorithm~\ref{alg:HJ}.\footnote{
	When $\nabla_{\bm{a}} Q_{\theta}(x_{j}, a_j) = 0$, $ \frac{\nabla_{\bm{a}} Q_{\theta}(x_{j}, a_j)}{|\nabla_{\bm{a}} Q_{\theta}(x_{j}, a_j)|}$ is replaced by an arbitrary vector with norm $1$ of the same size.
}

\begin{algorithm}[tb]
	\caption{Hamilton--Jacobi DQN}
	\label{alg:HJ}
	\begin{algorithmic}
		\STATE Initialize Q-function $Q_\theta$ with random weights $\theta$, and target Q-function $Q_{\theta^-}$ with weights $\theta^- = \theta$;
		\STATE Initialize replay buffer with fixed capacity;
		\FOR{episode $=1$ {\bfseries to} $M$}
		\STATE Randomly sample initial state-action pair $(x_0, a_0)$;
		\FOR{$k=0$ {\bfseries to} $K$}
		\STATE   Execute action $a_k$ and observe reward $r_k$ and the next state $x_{k+1}$;
		\STATE Store $(x_k, a_k, r_k, x_{k+1})$ in buffer;
		\STATE Sample the random mini-batch $\{(x_j, a_j, r_j, x_{j+1})\}$ from buffer;
		\STATE Set $y_j^- := h r_j + (1-\gamma h) Q_{\theta^-} \big (x_{j+1}, a_j' \big )$ $\forall j$ where $a_j':= a_j + h L \frac{\nabla_{\bm{a}} Q_{\theta}(x_{j}, a_j)}{|\nabla_{\bm{a}} Q_{\theta}(x_{j}, a_j)|}$;
		\STATE Update $\theta$ by minimizing $\sum_j (y_j^- - Q_\theta (x_j, a_j))^2$;
		\STATE Update $\theta^- \leftarrow (1-\alpha) \theta^- + \alpha \theta$ for $\alpha \ll 1$;
		\STATE Set the next action as $a_{k+1} := a_k + h L\frac{\nabla_{\bm a} Q_\theta(x_k,a_k)}{|\nabla_{\bm a}Q_\theta(x_k,a_k)|}+\varepsilon$, where $\varepsilon\sim N(0,\sigma^2I_m)$;
		\ENDFOR
		\ENDFOR
	\end{algorithmic}
\end{algorithm}

\subsection{Discussion}

We now discuss a few notable features of HJ DQN with regard to existing works:

{\bf No use of parameterized policies.} 
Most of model-free deep RL algorithms for continuous control use  actor-critic methods~\cite{Lillicrap2016, Haarnoja2018, Fujimoto2018, Tessler2019} or policy gradient methods~\cite{Schulman2015, Gu2016} to deal with  continuous action spaces. 
In these methods, by parametrizing policies, the policy improvement step is performed in the space of network weights. By doing so, they avoid solving possibly complicated optimization problems over the policy or action spaces. However, these methods are subject to the issue of being stuck at local optima in the policy (parameter) space due to the use of gradient-based algorithms, as pointed out in the literature regarding policy gradient/search~\cite{Kohl2004, Levine2013,Fazel2018global} 
and actor-critic methods~\cite{Silver2014}. Moreover, it is reported that the policy-based methods are sensitive to hyperparameters \cite{Quillen2018deep}. Departing from these algorithms, HJ DQN is a value-based method for continuous control without requiring the use of an actor or a parameterized policy.
Previous value-based methods for continuous control (e.g., \cite{Ryu2020}) have a computational challenge in finding a greedy action, which requires a  solution to a nonlinear program. 
Our method avoids numerically optimizing Q-functions over the continuous action space through the use of the optimal control \eqref{opt_con}. 
This is a notable benefit of the proposed HJB framework.  


%
%
%

{\bf Continuous-time control.} Many existing RL methods for continuous-time dynamical systems have been designed for linear systems~\cite{Palanisamy2015, Bian2016-2, Vamvoudakis2017} or control-affine systems~\cite{Jiang2015, Bhasin2013, Modares2014,Vamvoudakis2010}, in which value functions and optimal policies can be represented in a simple form. 
For general nonlinear systems, Hamilton--Jacobi--Bellman equations have been considered as the optimality equations for state-value functions $v(\bm{x})$~\cite{Doya2000, Munos2000, Dayan1996, Ohnishi2018}. 
Unlike these methods,
our method uses variant of Q-function and thus benefits from modern deep RL techniques developed in the literature on DQN.
Moreover, as opposed to discrete-time RL methods, it does not discretize or approximate the system dynamics and has the flexibility of choosing the sampling interval $h$ in its algorithm design, without needing a sophisticated ODE discretization method.

\subsection{Smoothing}\label{subsec:smoothing}

A potential defect of our Lipschitz constrained control setting is that 
the rate of change in action has a constant norm $L \frac{\nabla_{\bm{a}}Q (x^\star, a^\star)}{|\nabla_{\bm{a}}Q (x^\star, a^\star)|}$. 
This is also observed 
in Algorithm \ref{alg:HJ}, where the action is updated by $hL\frac{\nabla_{\bm{a}}Q_\theta(x_j,a_j)}{|\nabla_{\bm{a}}Q_\theta(x_j,a_j)|}$. 
Therefore, the magnitude of fluctuations in action is always fixed as $hL$, which may lead to the oscillatory behavior of action. 
Such oscillatory behaviors are not uncommon in optimal control (e.g., bang-bang solutions).
To alleviate this potential issue, one may introduce an additional smoothing process when updating action. 
Inspired by \cite{abu2005nearly}, we modify the term $\frac{\nabla_{\bm a} Q_\theta(x_j,a_j)}{|\nabla_{\bm a}Q_\theta(x_j,a_j)|}$ by multiplying a smoothing function. Instead of using $hL\frac{\nabla_{\bm{a}}Q_\theta(x_j,a_j)}{|\nabla_{\bm{a}}Q_\theta(x_j,a_j)|}$ in the update of action, we suggest to use 
\begin{equation} \nonumber
hL \frac{\phi(|\nabla_{\bm{a}}Q_{\theta}(x_j,a_j)|)\nabla_{\bm{a}}Q_{\theta}(x_j,a_j)}{|\nabla_{\bm{a}}Q_{\theta}(x_j,a_j)|},
\end{equation}
where $\phi:[0,+\infty)\to[0,1]$ is an increasing function with $\phi(0)=0$ and $\lim_{r\to\infty}\phi(r)=1$. A typical example of such a function $\phi$ is $\phi(r)=\tanh\left(\frac{r}{L}\right)$ or $\phi(r)=\frac{r}{L+r}$. This action update rule  is expected to remove the undesirable oscillatory behavior of action, as confirmed in Section~\ref{sec:exp}.

\section{Experiments}\label{sec:exp}

In this section, we present the empirical performance of our method on benchmark tasks as well as LQ problems. 
The source code of our HJ DQN implementation is available online~\footnote{\url{https://github.com/HJDQN/HJQ}}.

\subsection{Actor Networks vs. Optimal Control ODE}
\label{sec:mujoco}

 We choose deep deterministic policy gradient (DDPG)~\cite{Lillicrap2016} as a baseline to compare since it is another variant of DQN for continuous control. 
DDPG is an actor-critic method using separate actor networks while ours is a valued-based method that does not use a parameterized policy. Although there are state-of-the-art methods built upon DDPG,
such as TD3 \cite{Fujimoto2018} and SAC \cite{Haarnoja2018},
we focus on the comparison between ours and DDPG to examine whether the role of actor networks can be replaced by the optimal control characterized through our HJB equation. The hyperparameters used in the experiments are reported in Appendix~\ref{app:detail}.

\begin{figure*}
	\begin{subfigure}{1\textwidth}
		\centering
		\includegraphics[width=0.5\textwidth]{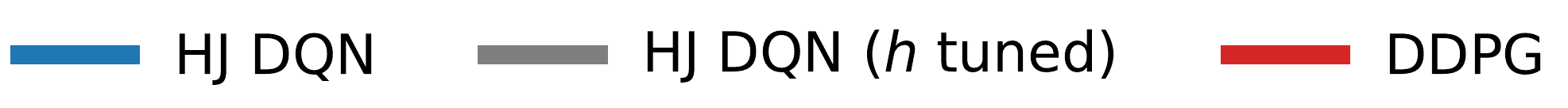}
	\end{subfigure}
	\begin{subfigure}{.24\textwidth}
		\centering
		\includegraphics[height=.172\textheight]{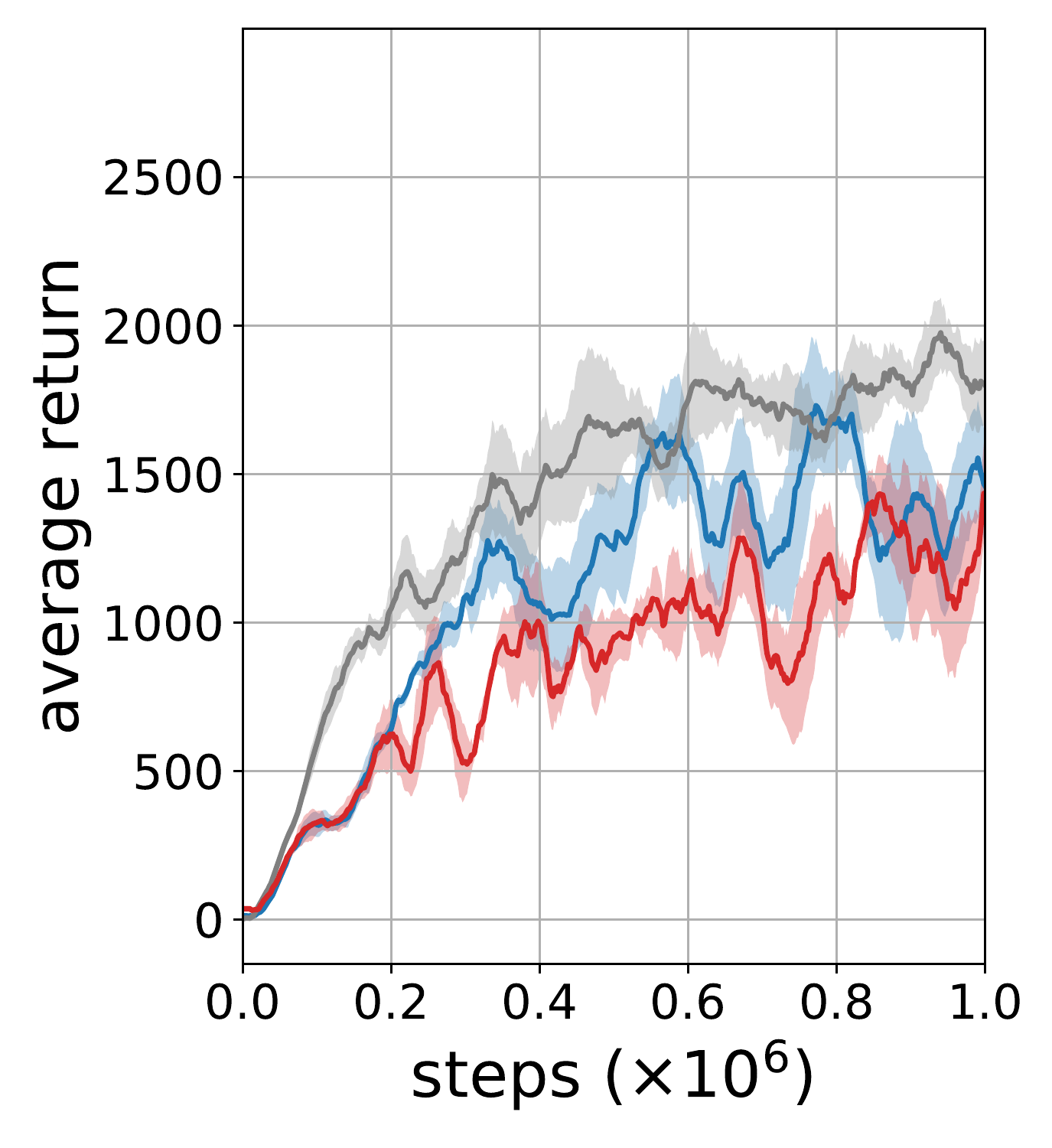}  
		\caption{Hopper-v2}
		\label{fig:1a}
	\end{subfigure}
	\hspace{0.05cm}
	\begin{subfigure}{.24\textwidth}
		\centering
		\includegraphics[height=.172\textheight]{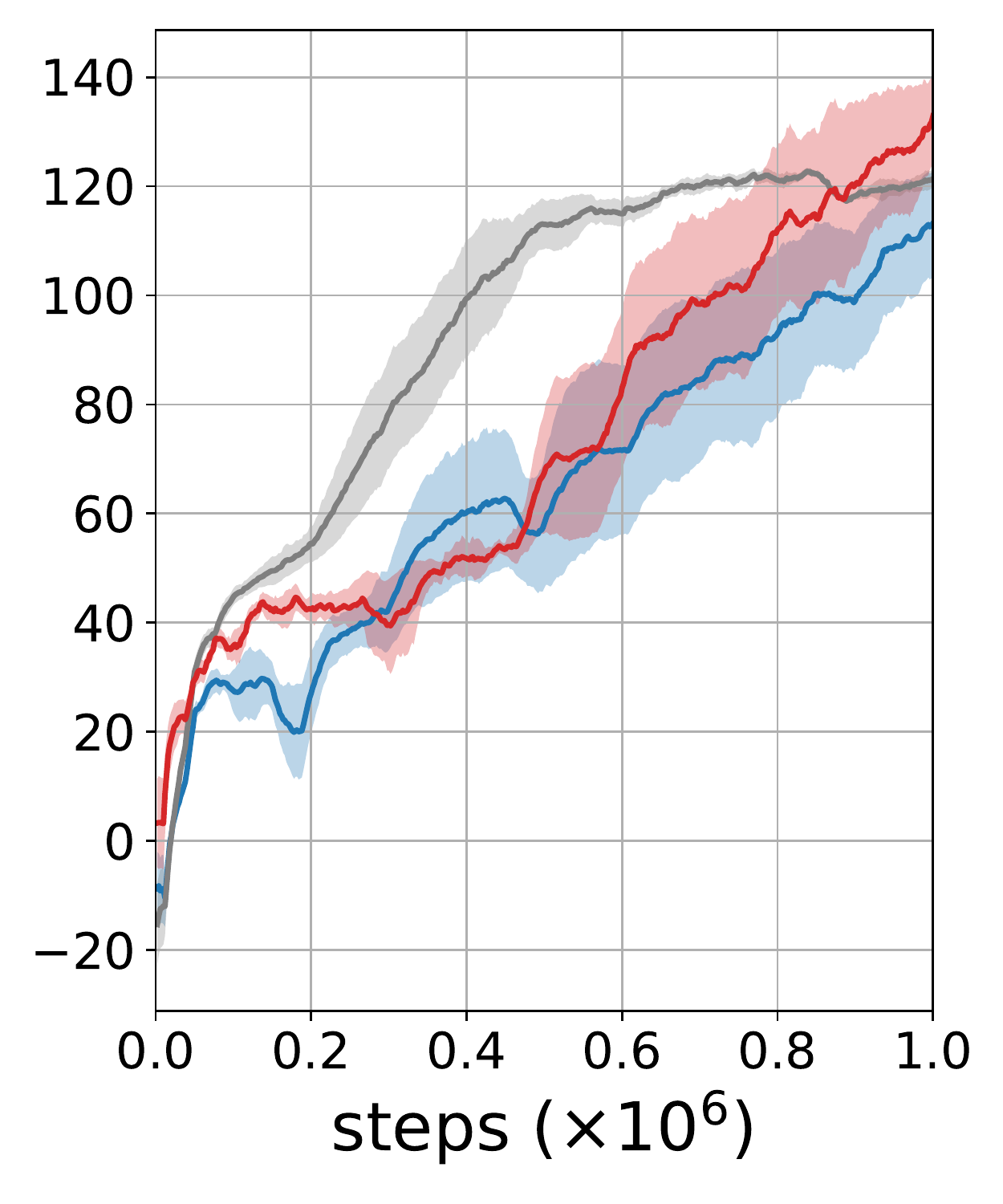}  
		\caption{Swimmer-v2}
		\label{fig:1b}
	\end{subfigure}
	\hspace{-0.15cm}
	\begin{subfigure}{.24\textwidth}
		\centering
		\includegraphics[height=.172\textheight]{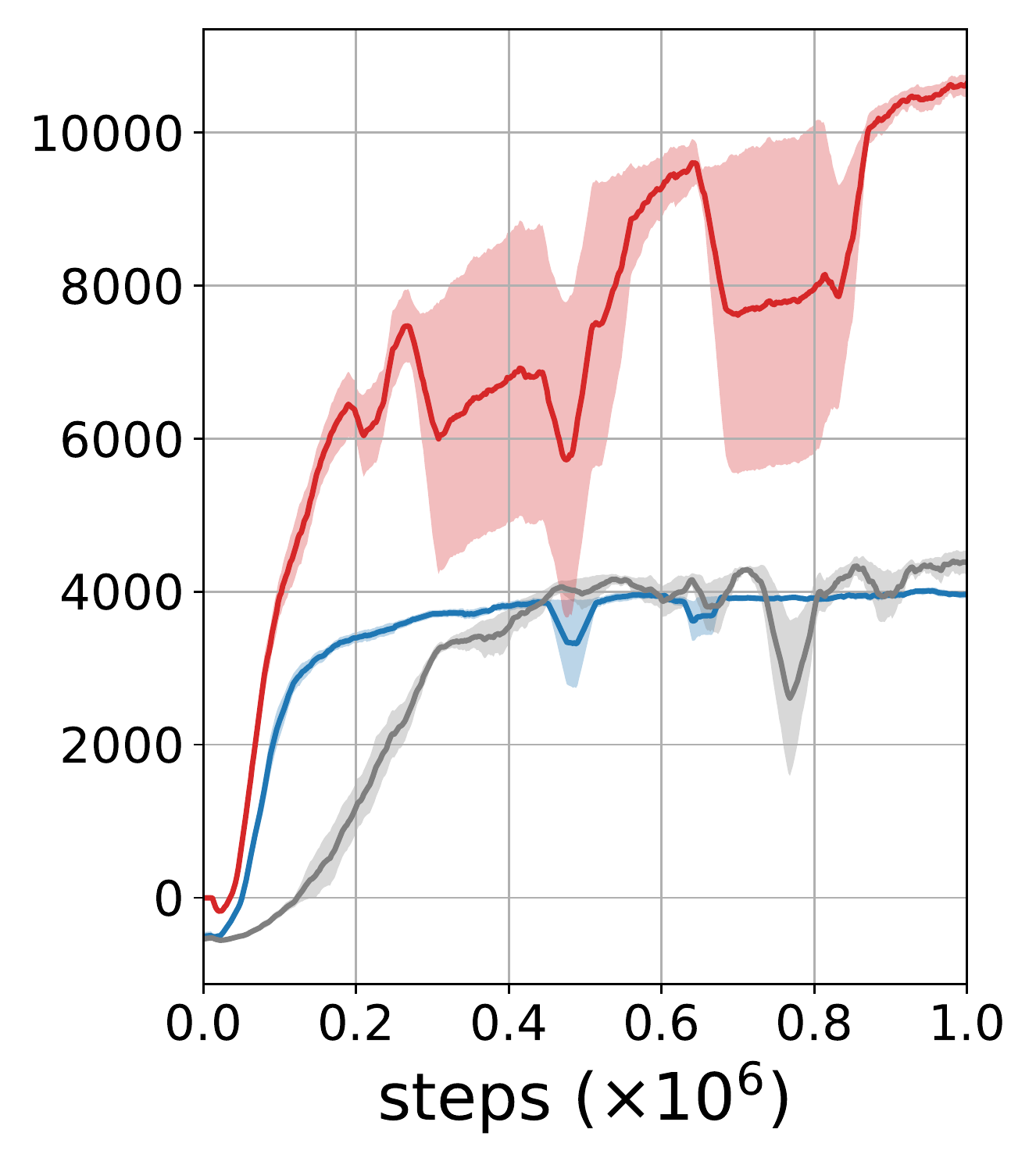}  
		\caption{HalfCheetah-v2}
		\label{fig:1c}
	\end{subfigure}
	\hspace{0.15cm}
	\begin{subfigure}{.24\textwidth}
		\centering
		\includegraphics[height=.172\textheight]{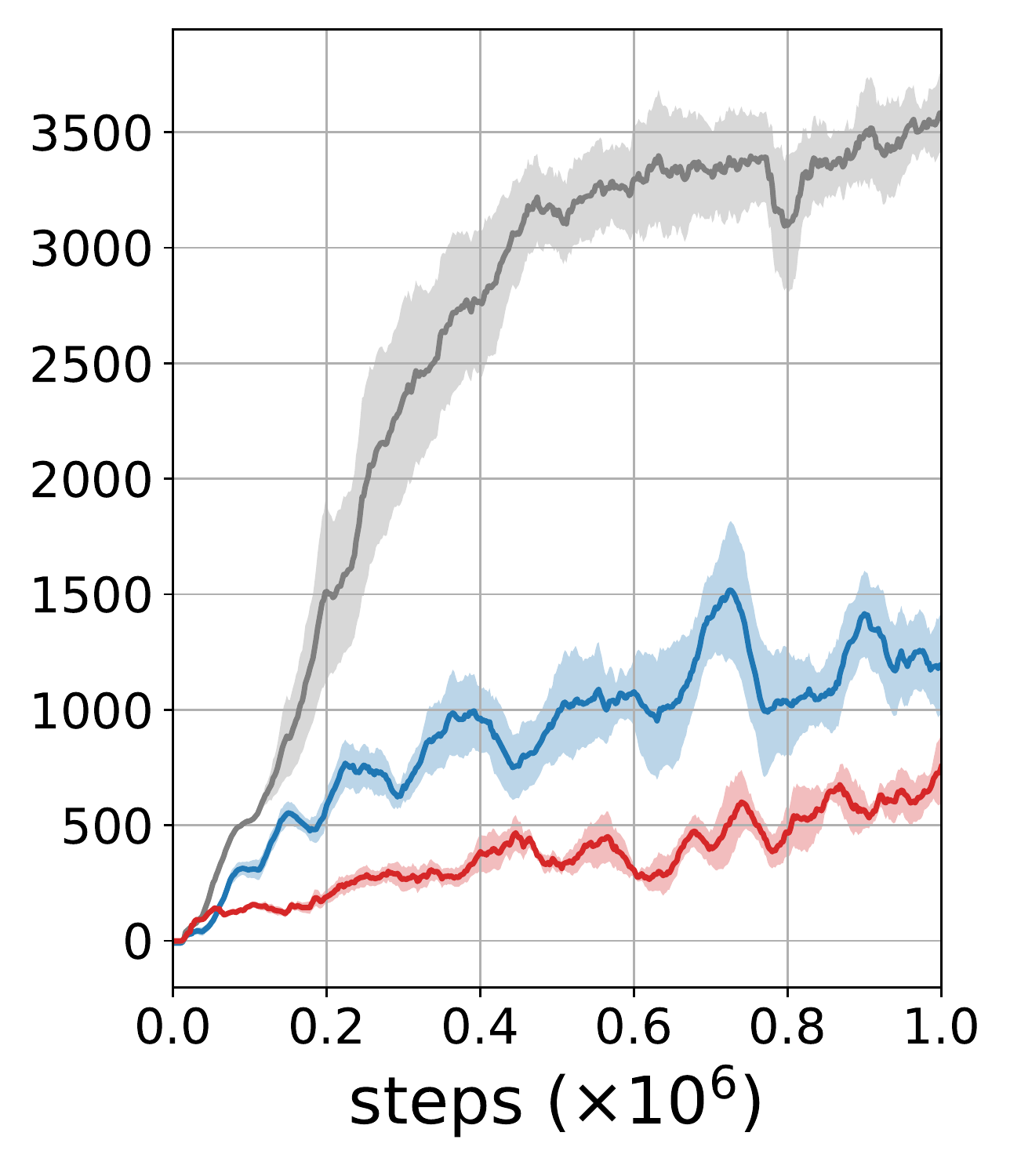}  
		\caption{Walker2d-v2}
		\label{fig:1d}
	\end{subfigure}
	\caption{Learning curves for the OpenAI gym continuous control tasks.}
	\label{fig:mujoco}
\end{figure*}

\begin{figure*}
	\begin{subfigure}{.48\textwidth}
		\centering	
		\includegraphics[height=.25\textheight]{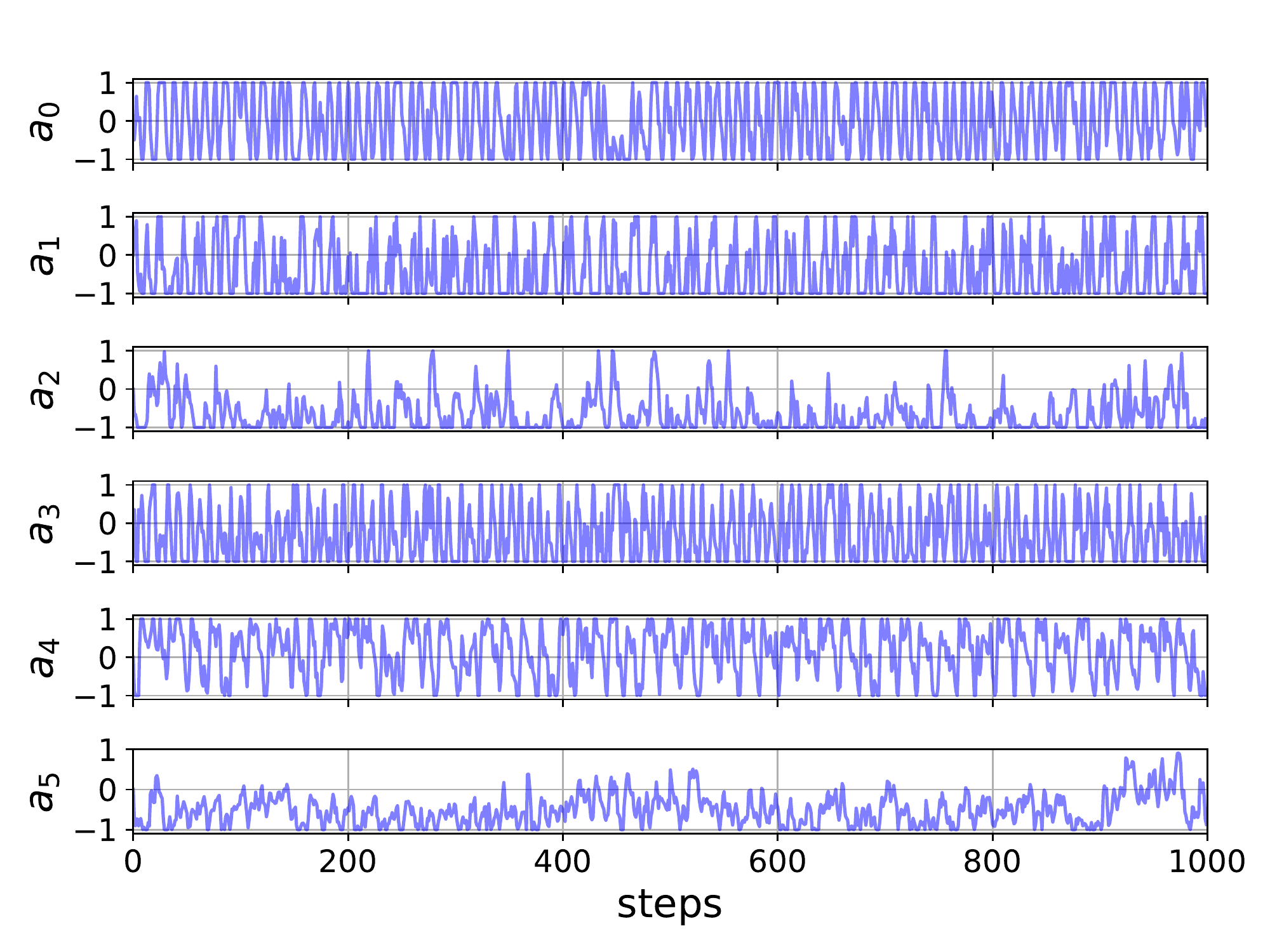}
		\caption{HJ DQN}
		\label{ctrl : a}
	\end{subfigure}
	\begin{subfigure}{.48\textwidth}
		\centering	
		\includegraphics[height=.25\textheight]{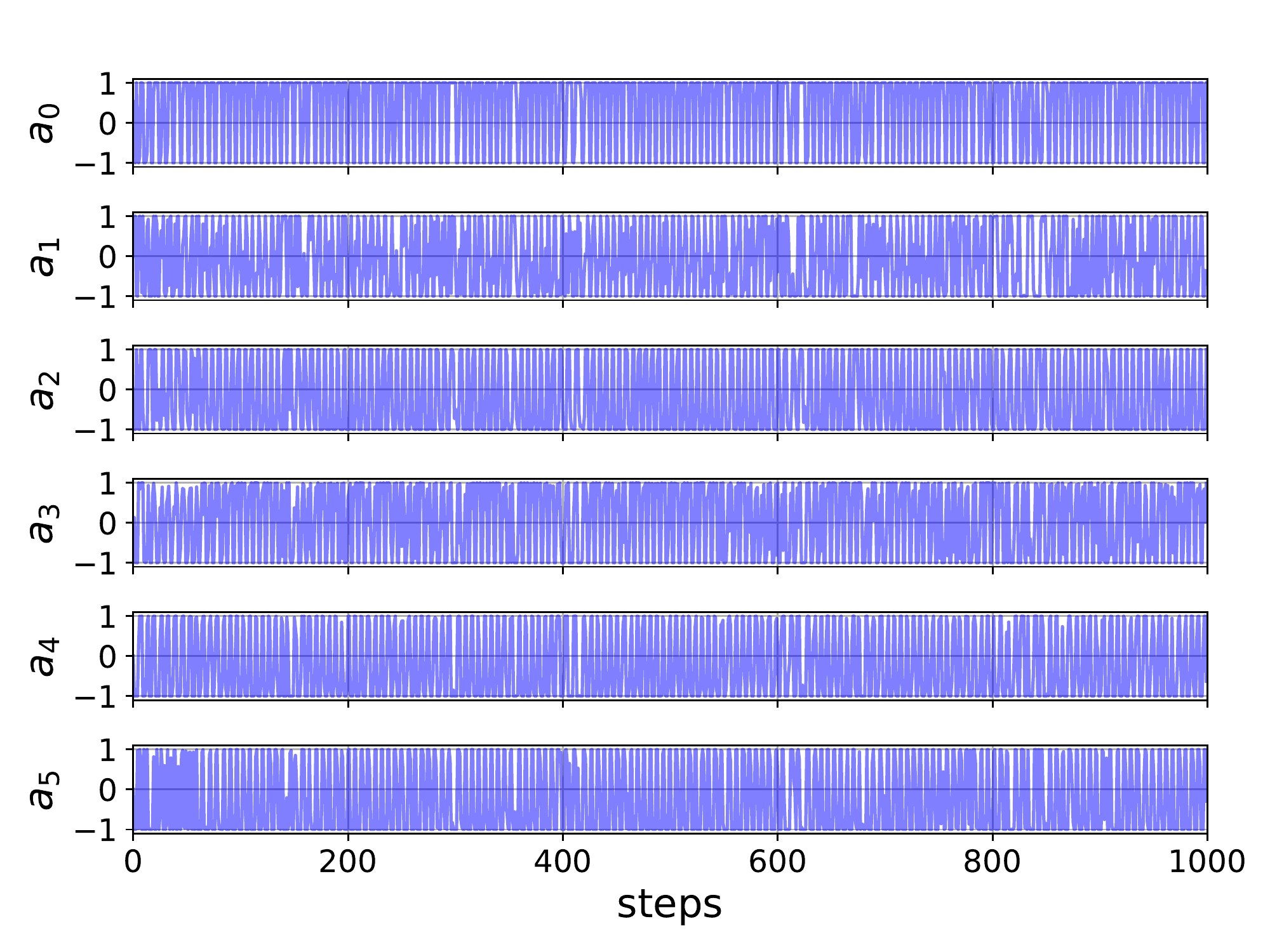}
		\caption{DDPG}
		\label{ctrl : b}
	\end{subfigure}
	\caption{Action trajectories obtained by HJ DQN and DDPG for HalfCheetah-v2.}
	\label{action_half}
\end{figure*}

We consider continuous control benchmark tasks in OpenAI gym~\cite{Brockman2016} simulated by MuJoCo engine~\cite{Todorov2012}.
Figure \ref{fig:mujoco} shows the learning curves for both methods, each of which is tested with five different random seeds for 1 million steps. The solid curve represents the average of returns over 20 consecutive evaluations, while the shaded regions represent  half a standard deviation of the average evaluation over five trials. 
As shown in Figure~\ref{fig:mujoco}, the performance of our method is comparable to that of DDPG when the default sampling interval  is used. 
Ours outperforms DDPG on Walker2d-v2 while the opposite result is observed in the case of HalfCheetah-v2. 
As sampling interval $h$ is a hyperparameter of Algorithm~\ref{alg:HJ}, 
we also identify an optimal $h$ for each task, other than the default sampling interval. When we test the different sampling interval, we also tune the learning rate $\alpha$, as suggested in \cite{pmlr-v97-tallec19a}. Precisely, when the sampling interval is multiplied by a constant from the default interval, the learning rate is also multiplied by the same constant. Except the HalfCheetah-v2, the final performances or learning rate are improved, compared to the default sampling interval. Overall, the results indicate that 
actor networks may be replaced by the ODE characterization~\eqref{opt_con} of optimal control obtained using our HJB framework.  
Without using actor networks, our method has clear advantages over DDPG in terms of hyperparameter tuning and computational burden.


In Figure \ref{action_half}, we report the action trajectories obtained by HJ DQN and DDPG for HalfCheetah-v2.  The action trajectories obtained by HJ DQN is less oscillating compared to DDPG. 
This confirms the fact that oscillations in action are not uncommon in optimal control. 
In this particular case of HalfCheetah-v2, where DDPG outperforms HJ DQN,  we suspect that fast changes in action may be needed for good performance. Oscillatory actions may be beneficial to some control tasks. 
%

\subsection{Linear-Quadratic Problems}

\begin{figure*}[t!]
	\centering
	\begin{subfigure}{.45\textwidth}	
		\includegraphics[height=.25\textheight]{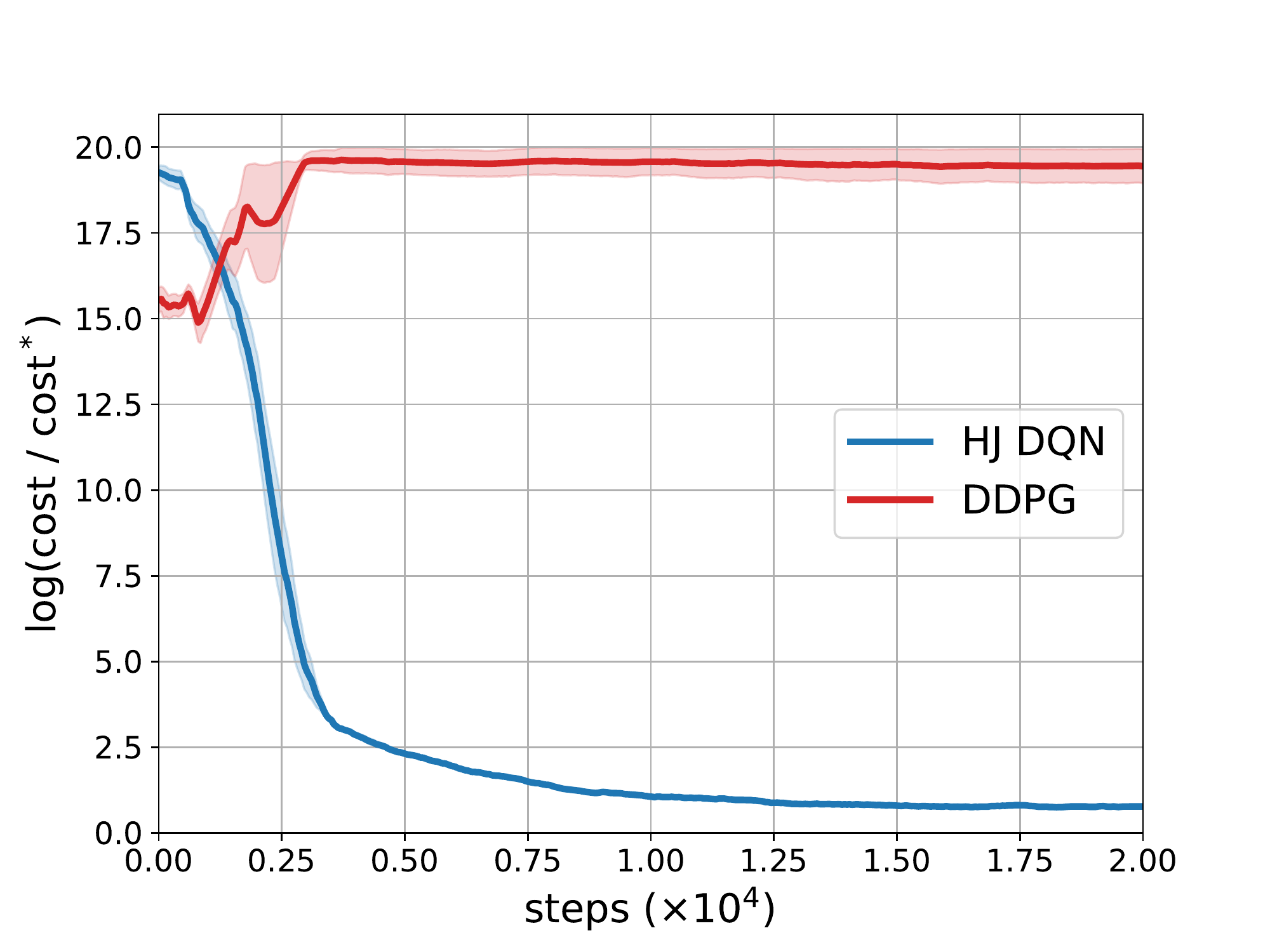}
		\caption{}
		\label{fig:lqr_ddpg}
	\end{subfigure}
	\begin{subfigure}{.45\textwidth}
		\includegraphics[height=.25\textheight]{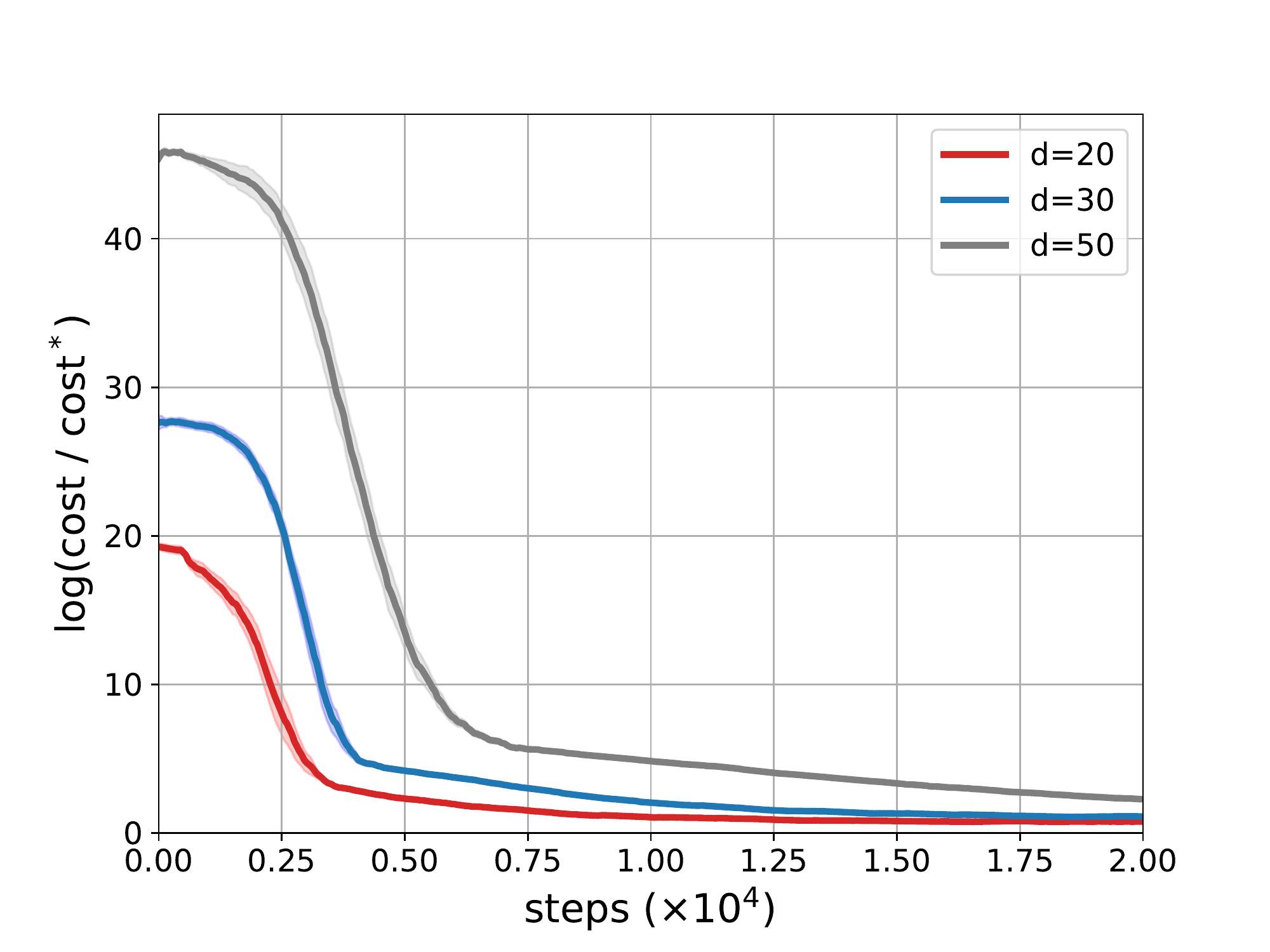}
		\caption{}
		\label{fig:lqr_dim}
	\end{subfigure}
	\caption{Learning curves for the LQ problem: (a) the comparison between HJ DQN and DDPG, and (b) the effect of problem sizes.}
	\label{fig:lqr}
\end{figure*}

We now consider a classical LQ problem with system dynamics
\[\dot{x}(t) = Ax(t)+Ba(t),\quad t>0,\quad x(t),\,a(t)\in\bbr^d,\]
and reward function
\[
r(\bm{x},\bm{a}) = -(\bm{x}^\top Q\bm{x}+\bm{a}^\top R\bm{a}),
\]
where $Q = Q^\top \succeq 0$ and $R = R^\top \succ 0$.
Note that our method solves a slight different problem due to the Lipschitz constraint on controls. 
Thus, the control learned by our method must be suboptimal. 

Each component of the system matrices $A \in \mathbb{R}^{d \times d}$ and $B \in \mathbb{R}^{d \times d}$ was generated uniformly from $[-0.1,0.1]$ and $[-0.5,0.5]$, respectively. The produced matrix $A$ has an eigenvalue with positive real part, and therefore the system is unstable. The discount factor $\gamma$ and Lipschitz constant $L$ are set to be $e^{-\gamma h} = 0.99999$ and $L=10$. We first compare the performance of HJ DQN with DDPG for the case of $d=20$ and report the results in Figure \ref{fig:lqr_ddpg}. 
The learning curves are plotted in the same manner as the ones in Section \ref{sec:mujoco}. 
The $y$-axis of each figure is the log of the ratio between the actual cost and the optimal cost. 
Therefore, the curve approaches the $x$-axis as the performance improves. 
The result implies that the DDPG cannot reduce the cost at all, whereas HJ DQN successfully learns an effective (suboptimal) policy. 
In Figure \ref{fig:lqr_dim}, we present the learning curves for HJ DQN with different system sizes. 
Although  learning speed is affected by the problem size, HJ DQN can successfully solve the LQ problem with high-dimensional systems ($d = 50$). Moreover, it is observed that the standard deviations over trials are relatively small, and the learning curves have almost no variation over trials after approximately $10^4$ steps.

\subsection{Ablation Study}

We make ablations and modifications to HJ DQN to understand the contribution of each component.
Figure~\ref{fig:design} presents the results for the following design evaluation experiments.

%

\begin{figure*}
	\begin{subfigure}{.32\textwidth}
		\centering
		\includegraphics[height=0.23\textheight]{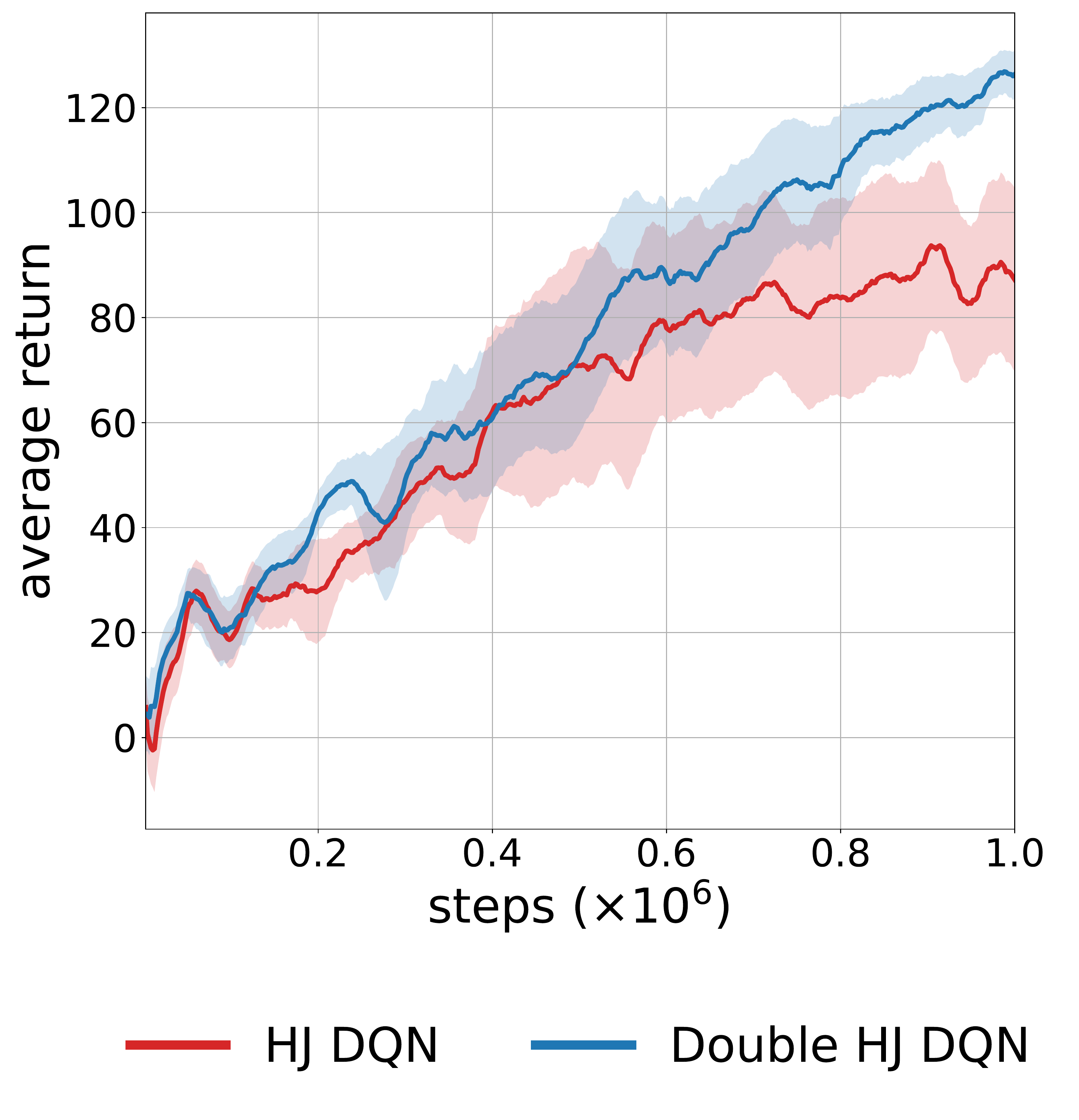}  
		\caption{Double Q-learning}
		\label{fig:2a}
	\end{subfigure}
	\hspace{0.1cm}
	\begin{subfigure}{.32\textwidth}
		\centering
		\includegraphics[height=0.23\textheight]{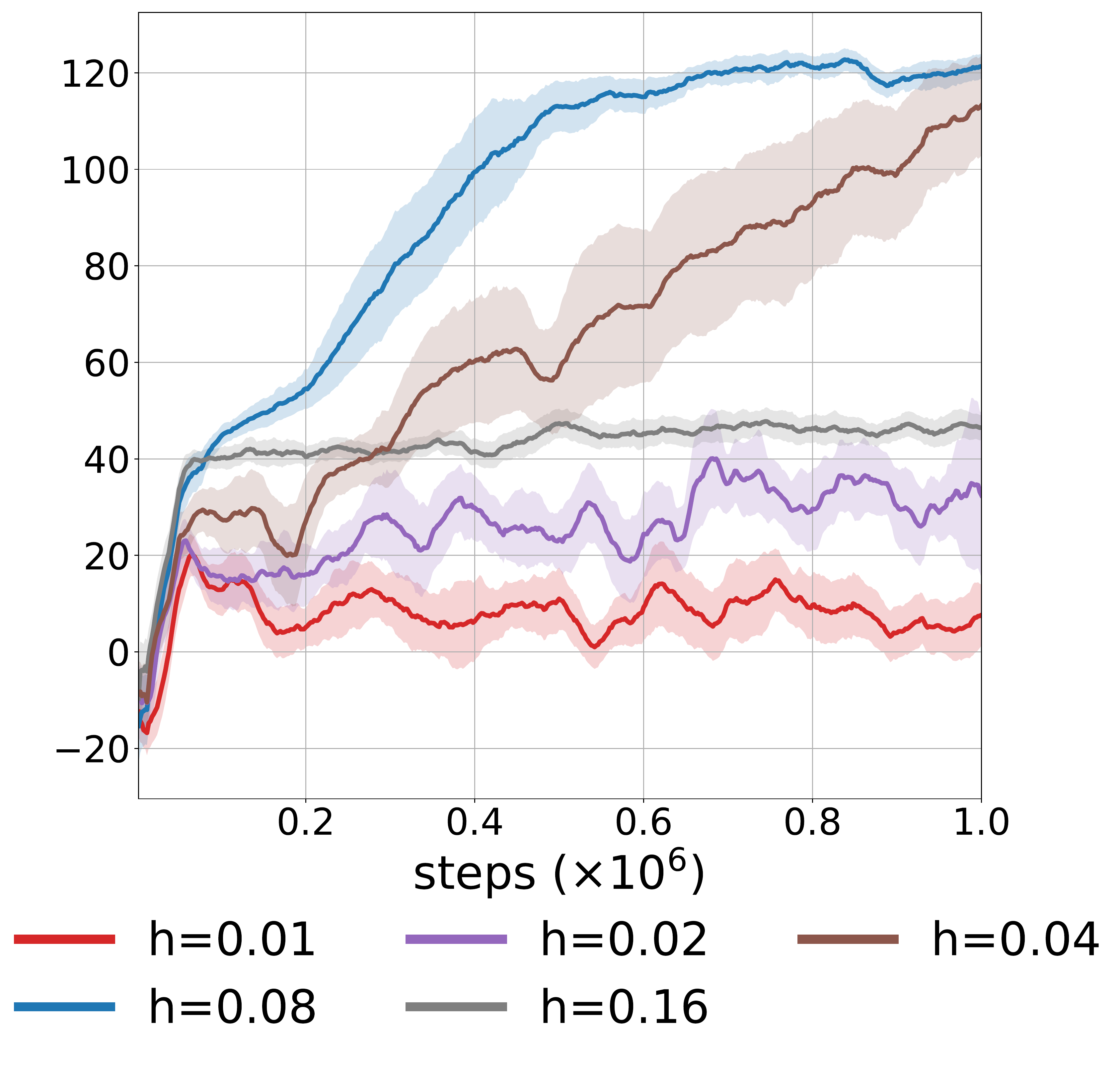}  
		\caption{Sampling interval}
		\label{fig:2b}
	\end{subfigure}
	\begin{subfigure}{.32\textwidth}
		\centering
		\includegraphics[height=0.23\textheight]{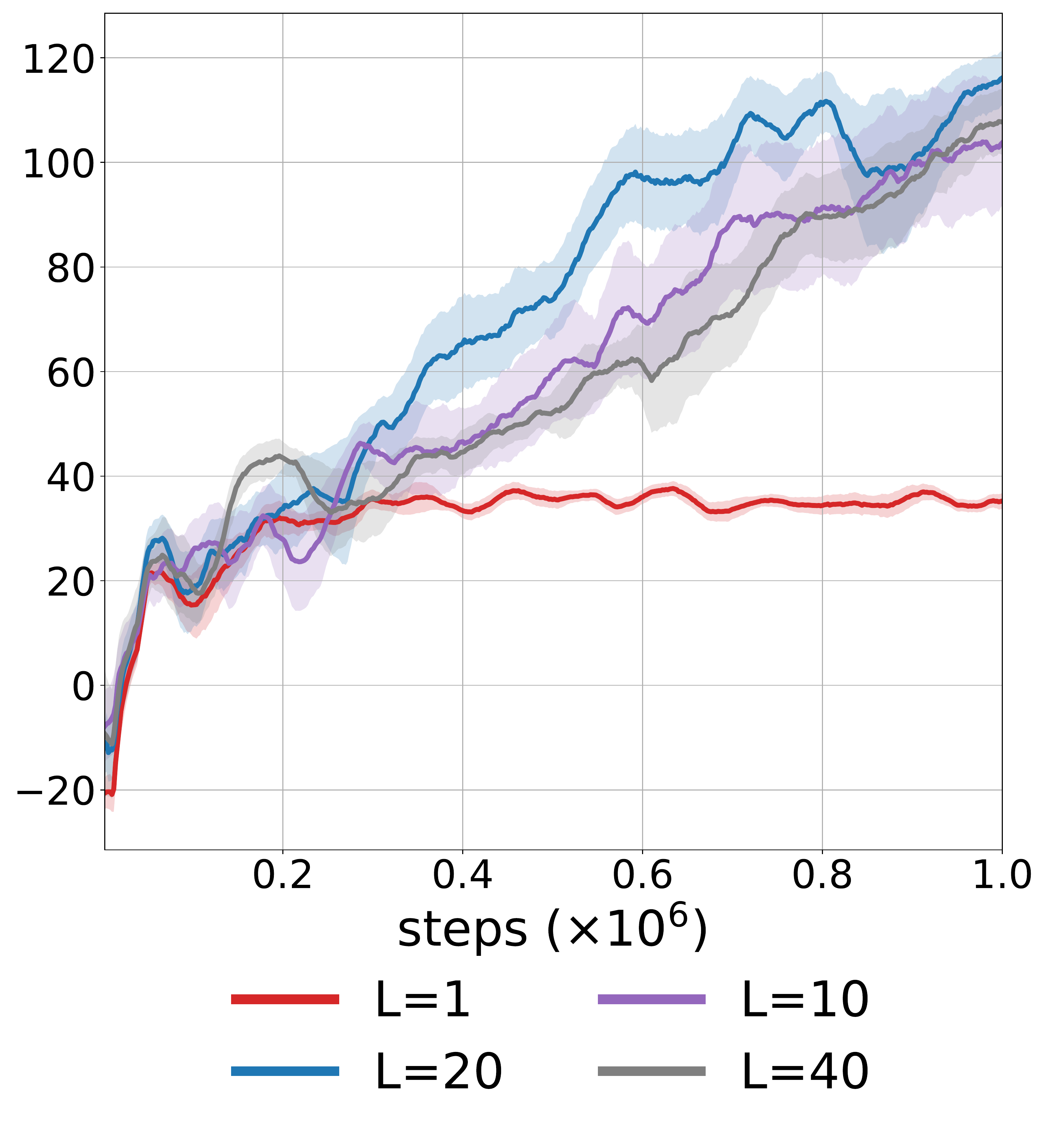}  
		\caption{Control constraint}
		\label{fig:2c}
	\end{subfigure}
	\caption{Results of the ablation study using the Swimmer-v2
		with respect to (a) double Q-learning,    (b) sampling interval $h$, and (c) control constraint.}
	\label{fig:design}
\end{figure*}

{\bf Double Q-learning.}
We first modify our algorithm to test whether double Q-learning contributes to the performance of our algorithm, as in DQN. 
Specifically, when selecting actions to update the target value, we instead use 
$b_j := L \frac{\nabla_{\bm{a}} Q_{\theta^-} (x_j, a_j)}{|\nabla_{\bm{a}} Q_{\theta^-} (x_j, a_j)|}$ to remove the effects of double Q-learning.
Figure~\ref{fig:design} (a) shows that double Q-learning  improves the final performance. This observation is consistent with the effect of double Q-learning in DQN. 
Moreover, double Q-learning reduces the variance of the average return, indicating its contribution to the stability of our algorithm. 

%
%

{\bf Sampling interval.}
To understand the effect of sampling interval $h$, we run our algorithm with multiple values of $h$. As we mentioned before, we also adjust the learning rate $\alpha$ according to the sampling interval. As shown in Figure~\ref{fig:design} (b),  the final performance and  learning speed increase as $h$ varies from 0.01 to 0.08 and the final performance decreases as $h$ varies from 0.08 to 0.16. 
When $h$ is too small, each episode has too many sampling steps; thus, the network is trained in a small number of episodes given fixed total steps. 
This limits exploration, thereby decreasing the performance of our algorithm. 
On the other hand, as Proposition~\ref{prop:diff} implies, 
the target error increases with  sampling interval $h$. 
This error is dominant in the case of large $h$.
Therefore, there exists an optimal sampling interval ($h = 0.08$ in this task) that presents the best performance.

{\bf Control constraint.}
Recall that admissible controls satisfy the constraint $|\dot{a}(t)| \leq L$.
The parameter $L$ can be derived from specific control problems or considered as a design choice. We consider the latter case and display the effect of $L$ on the learning curves in Figure~\ref{fig:design} (c).
The final reward is the lowest, compared to others,
in the case of $L = 1$ because the set of admissible controls is too small to allow rapid changes in control signals. HJ DQN, with large enough $L$ ($\geq 10$), presents a similar learning speed and performance. The final performance and learning speed slightly decrease as $L$ varies from $20$ to $40$. 
This is due to too large variation and frequent switching in action values, prohibiting a consistent improvement of Q-functions.

{\bf Smoothing.}
\begin{figure*}[t]
	\centering
	\includegraphics[height=0.3\textheight]{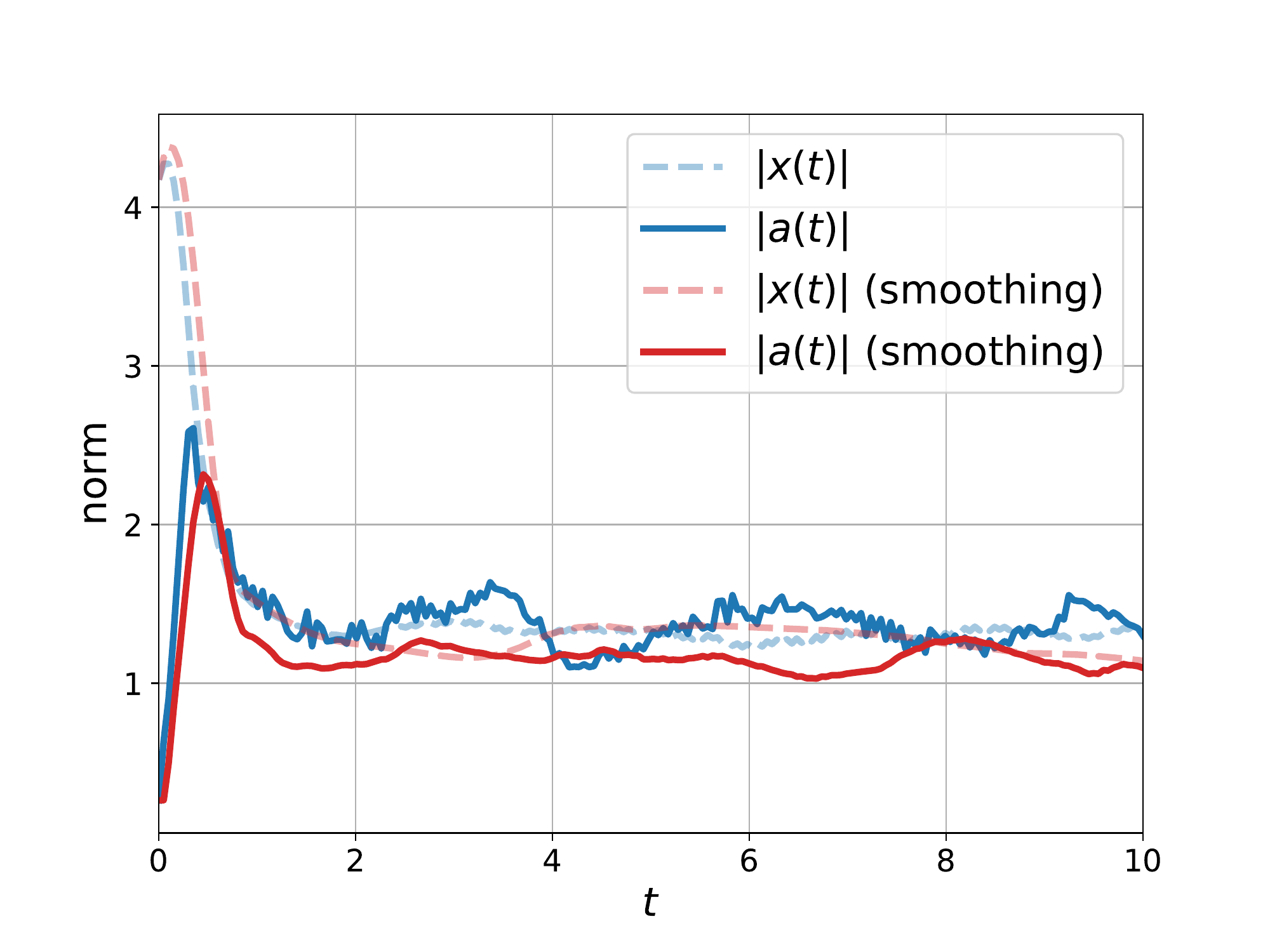}
	\caption{Effect of smoothing on the 20-dimensional LQ problem.}
	\label{fig:smoothing}
\end{figure*}
Finally, we present the effect of the smoothing process introduced in Section~\ref{subsec:smoothing}. 
Figure \ref{fig:smoothing} shows $|x(t)|$ and $|a(t)|$ generated by the control learned 
with and without smoothing on  the 20-dimensional LQ problem.
Here, $\phi(r) = \tanh\left(\frac{r}{L}\right)$ is chosen as the smoothing function. As expected, with no smoothing process, the action trajectory shows wobbling oscillations (solid blue line). However, when the smoothing process is applied, the action trajectory has no such  undesirable oscillations and presents a smooth behavior (solid red line). 
Regarding $|x(t)|$, the smoothing process  has only a small effect. 
Therefore, the smoothing process can eliminate oscillations in action without significantly affecting the state trajectory.
%



\section{Conclusions}\label{sec:conclusion}

We have presented a new theoretical and algorithmic framework that extends DQN to continuous-time deterministic optimal control for continuous action space. 
A novel class of HJB equations for Q-functions has been derived and 
used to construct a Q-learning method for continuous-time control.
We have shown the theoretical convergence properties of this method.
For practical implementation, we have combined the HJB-based method with DQN, resulting in a simple algorithm that solves continuous-time control problems without an actor network.
Benefiting from our theoretical analysis of the HJB equations,
this model-free off-policy algorithm 
does not require any numerical optimization for selecting greedy actions.
The result of our experiments indicates that actor networks in DDPG may be replaced by our optimal control simply characterized via an ODE, while reducing computational effort.
Our HJB framework may provide an exciting avenue for future research in continuous-time RL in terms of improving the exploration capability with maximum entropy methods, and exploiting the benefits of models with theoretical guarantees.


\appendix

\section{Viscosity Solution of the Hamilton--Jacobi Equations}\label{app:vis}

The Hamilton--Jacobi equation is a partial differential equation of the form
\begin{equation}\label{general_HJ}
F (\bm{z}, u(\bm{z}), \nabla_{\bm{z}} u (\bm{z})) = 0, \quad \bm{z} \in \mathbb{R}^k,
\end{equation}
where $F: \mathbb{R}^k \times \mathbb{R} \times \mathbb{R}^k \to \mathbb{R}$. A function $u: \mathbb{R}^k \to \mathbb{R}$ that solves the HJ equation is called a (strong) solution. However, such a strong solution exists only in limited cases. To consider a broad class of HJ equations, it is typical to adopt the concept of weak solutions. Among these, the \emph{viscosity solution} is the most relevant to dynamic programming and optimal control problems~\cite{Crandall1983, Bardi1997}. Specifically, under a technical condition, the viscosity solution is unique and corresponds to the value function of a continuous-time optimal control problem. In the following definition, $C(\bbr^k)$ and $C^1(\bbr^k)$ denote the set of continuous functions and the set of continuously differentiable functions respectively.
\begin{defn}\label{def:vis}
	A function $u\in C(\bbr^k)$ is called the \emph{viscosity solution} of \eqref{general_HJ} if it satisfies the following conditions:
	\begin{enumerate}
		\item For any $\phi\in C^1(\bbr^k)$ such that $u-\phi$ attains a local maximum at ${\bm z}_0$, 
		\[
		F({\bm z}_0, u({\bm z}_0),\nabla_{{\bm z}}\phi({\bm z}_0))\le 0; 
		\]
		\item For any $\phi\in C^1(\bbr^k)$ such that $u-\phi$ attains a local minimum at ${\bm z}_0$, 
		\[
		F({\bm z}_0, u({\bm z}_0),\nabla_{{\bm z}}\phi({\bm z}_0))\ge 0.
		\]
	\end{enumerate}
\end{defn}
Note that the viscosity solution does not need to be differentiable. In our case, the HJB equation \eqref{HJB}
\[
\gamma Q({\bm x},{\bm a})-\nabla_{\bm x} Q({\bm x},{\bm a}) \cdot f({\bm x},{\bm a})-L|\nabla_{\bm a}Q({\bm x},{\bm a})|-r({\bm x},{\bm a})=0
\]
can be expressed as \eqref{general_HJ} with
\[
F({\bm z},q,{\bm p}) = \gamma q - {\bm p}_1\cdot f({\bm z}) - L |{\bm p}_2| - r({\bm z}),
\]
where ${\bm z} = ({\bm x},{\bm a})\in \bbr^{n} \times \bbr^m$ and ${\bm p} = ({\bm p}_1,{\bm p}_2)\in \bbr^{n} \times \bbr^m$. We can show that the HJB equation admits a unique viscosity solution, which coincides with the optimal Q-function.

\begin{theorem}
	Suppose that Assumption~\ref{ass:bl} holds.\footnote{Assumption~\ref{ass:bl} can be relaxed by using a modulus associated with each function as in Chapter III.1--3 in \cite{Bardi1997}.} Then, the optimal continuous-time Q-function is the unique viscosity solution to the HJB equation~\eqref{HJB}.
\end{theorem}
\begin{proof}
	First, recall that our control trajectory satisfies the constraint $|\dot{a}| \leq L$. Therefore, our dynamical system can be written in the following extended form:
	\[
	\dot{x}(t) = f(x(t),a(t)),\quad \dot{a}(t) = b(t),\quad t>0, \quad |b(t)|\le L,
	\]
	by viewing $x(t)$ and $a(t)$ as state variables. More precisely,	the dynamics of the extended state variable  $z(t)=(x(t),a(t))$ can be written as
	\begin{equation}\label{sys_ext}
	\dot{z}(t) = G(z(t),b(t)), \quad t>0, \quad |b(t)|\le L,
	\end{equation}
	where $G({\bm z},{\bm b}) = (f({\bm z}),{\bm b})$. Applying the dynamic programming principle to the Q-function, we have
	\begin{equation}\nonumber
	Q({\bm z})= \sup_{|b(t)|\le L} \bigg \{ \int_t^{t+h} e^{-\gamma (s-t)} r(z(s))\, \d s + \:  e^{-\gamma h} Q(z(t+h))
	\mid z(t) = {\bm z}\bigg \}.
	\end{equation}
	The remaining proof is almost the same as the proof of Proposition 2.8, Chapter 3 in \cite{Bardi1997}. However, for the self-completeness of the paper, we provide a detailed proof. In the following, we show that the Q-function satisfies the two conditions in Definition~\ref{def:vis}.
	
	First, let $\phi\in C^1(\bbr^{n+m})$ such that $Q-\phi$ attains a local maximum at ${\bm z}$. Then, there exists $\delta>0$ such that $Q({\bm z})-Q({\bm z}')\ge \phi({\bm z})-\phi({\bm z}')$ for $|{\bm z}'-{\bm z}|<\delta$. Since $f$ and $r$ are bounded Lipschitz continuous, there exists $h_0>0$, which is independent of $b(s)$, such that $|z(s)-{\bm z}|\le \delta$, $|r(z(s))-r({\bm z})|\le C(s-t)$ and $|f(z(s))-f({\bm z})|\le C(s-t)$ for $t\le s\le t+h_0$, where $z(s)$ is a solution to \eqref{sys_ext} for $s\ge t$ with $z(t)={\bm z}$. Now, the dynamic programming principle for the Q-function implies that, for any $0<h<h_0$ and $\e>0$, there exists $b(s)$ with $|b(s)|\le L$ such that
	\[
	Q({\bm z}) \le \int_t^{t+h} e^{-\gamma (s-t)}r(z(s))\,\d s+e^{-\gamma h}Q(z(t+h)) +h\e,
	\]
	where $z(s)$ is now a solution to \eqref{sys_ext} with $z(t)={\bm z}$ under the particular choice of $b$. On the other hand, it follows from our choice of $h$ that
	\[
	\int_t^{t+h}e^{-\gamma(s-t)}r(z(s))\,\d s = \int_t^{t+h}e^{-\gamma(s-t)}r({\bm z})\,\d s+o(h),
	\]
	which implies that
	\[
	Q({\bm z})\le \int_t^{t+h}e^{-\gamma(s-t)}r({\bm z})\,\d s + e^{-\gamma h}Q(z(t+h))+h\e+o(h).
	\]
	Therefore, we have
	\begin{equation}\nonumber
	\begin{split}
	\phi({\bm z})-\phi(z(t+h))&\le Q({\bm z})-Q(z(t+h))\\
	&\le \int_t^{t+h} e^{-\gamma (s-t)}r({\bm z})\,\d s+(e^{-\gamma h}-1)Q(z(t+h)) +h\e+o(h).
	\end{split}
	\end{equation}
	Since the left-hand side of the inequality above is equal to $-\int_t^{t+h} \frac{\d}{\d s} \phi(z(s))\,\d s=-\int_t^{t+h} \nabla_{\bm z}\phi(z(s))\cdot G(z(s),b(s))\,\d s$, we obtain that
	\begin{align*}
	0&\le \int_t^{t+h} \nabla_{\bm z} \phi(z(s))\cdot G(z(s),b(s))\, \d s \\
	&\quad +\int_t^{t+h}e^{-\gamma (s-t)}r({\bm z})\,\d s+\left(e^{-\gamma h}-1\right)Q(z(t+h))+h\e+o(h)\\
	&\le \int_t^{t+h} (\nabla_{\bm x}\phi(z(s))\cdot f({\bm z})+L|\nabla_{\bm a}\phi(z(s))|)\,\d s\\
	&\quad +\int_t^{t+h}e^{-\gamma (s-t)}r({\bm z})\,\d s  +\left(e^{-\gamma h}-1\right)Q(z(t+h))+h\e+o(h).
	\end{align*}
	By dividing both sides by $h$ and letting $h\to0$, we conclude that
	\[
	\nabla_{\bm x}\phi({\bm z})\cdot f({\bm z}) +L|\nabla_{\bm a}\phi({\bm z})|+r({\bm z})-\gamma Q({\bm z})+\e\ge0.
	\]
	Since $\e$ was arbitrarily chosen, we confirm that the Q-function satisfies the first condition in Definition~\ref{def:vis}, i.e.,
	\[
	\gamma Q({\bm z})-\nabla_{\bm x}\phi({\bm z})\cdot f({\bm z}) -L|\nabla_{\bm a}\phi({\bm z})|-r({\bm z})\le 0.
	\]
	
	We now consider the second condition. Let $\phi \in C^1(\bbr^{n+m})$ such that $Q-\phi$ attains a local minimum at ${\bm z}$, i.e., there exists $\delta$ such that $Q({\bm z})-Q({\bm z}')\le \phi({\bm z})-\phi({\bm z}')$ for $|{\bm z}'-{\bm z}|<\delta$. Fix an arbitrary ${\bm b}\in\bbr^m$ such that $|{\bm b}|\le L$ and let $b(s)\equiv {\bm b}$ be a constant function. Let $z(s)$ be a solution to \eqref{sys_ext} for $s\ge t$ with $z(t)={\bm z}$ under the particular choice of  $b(s)\equiv {\bm b}$. Then, for sufficiently small $h$, $|z(t+h)-{\bm z}|\le \delta$, and therefore we have
	\begin{equation}\label{Qphi}
	\begin{split}
	Q({\bm z})-Q(z(t+h)) &\le \phi({\bm z})-\phi(z(t+h))= -\int_t^{t+h} \frac{\d}{\d s} \phi(z(s))\,\d s \\
	&=-\int_t^{t+h}\nabla_{\bm z} \phi (z(s))\cdot G(z(s),{\bm b})\,\d s.
	\end{split}
	\end{equation}	
	On the other hand, the dynamic programming principle yields
	\begin{align}
	\begin{aligned}\label{dynprin}
	Q({\bm z})-Q(z(t+h)) \ge\int_t^{t+h}e^{-\gamma (s-t)}r(z(s))\,\d s+(e^{-\gamma h}-1)Q(z(t+h)). 
	\end{aligned}
	\end{align}
	By \eqref{Qphi} and \eqref{dynprin}, we have
	\begin{equation}\nonumber
	(e^{-\gamma h}-1)Q(z(t+h))+\int_t^{t+h}e^{-\gamma (s-t)}r(z(s))\,\d s\le -\int_t^{t+h}\nabla_{\bm z} \phi (z(s))\cdot G(z(s),{\bm b})\,\d s.
	\end{equation}
	Dividing both sides by $h$ and letting $h\to0$, we obtain that
	\[
	-\gamma Q({\bm z}) + r({\bm z})\le -\nabla_{\bm z}\phi({\bm z}) \cdot (f({\bm z}),{\bm b}),
	\]
	or equivalently
	\[
	\gamma Q({\bm z})-\nabla_{\bm x}\phi({\bm z})\cdot f({\bm z}) -\nabla_{\bm a}\phi({\bm z})\cdot {\bm b}-r({\bm z})\ge 0.
	\]
	Since ${\bm b}$ was arbitrarily chosen from $\{{\bm b}\in\bbr^m: |{\bm b}|\le L\}$, we have
	\[
	\gamma Q({\bm z})-\nabla_{\bm x}\phi({\bm z})\cdot f({\bm z}) - L|\nabla_{\bm a}\phi({\bm z})|-r({\bm z})\ge 0,
	\]
	which confirms that the Q-function satisfies the second condition in Definition~\ref{def:vis}. Therefore, we conclude that the Q-function is a viscosity solution of the HJB equation~\eqref{HJB}. 
	
	Lastly, the uniqueness of the viscosity solution can be proved by using Theorem 2.12, Chapter 3 in \cite{Bardi1997}.
\end{proof}

\section{Proofs}\label{app:pf}

\subsection{Proposition~\ref{prop:lip}}

\begin{proof}
	Fix $(\bx, \bm a, t) \in  \bbr^n \times \bbr^m \times [0,T]$.
	Let $\e$ be an arbitrary positive constant. Then, there exists $a\in \mathcal{A}$ such that $\int_t^Tr(x(s),a(s))\,\d s+q(x(T))<v(\bx,t)+\e$,
	where $x(s)$ satisfies \eqref{sys} with $x(t)=\bx$ in the Carath\'eodory sense: $x(s)=\bx+\int_t^s f(x(\tau),a(\tau))\,\d \tau$.
	We now construct a new control $\tilde{a}\in\mathcal{A}$ as
	$\tilde{a}(s):= \bm a$ if $s=t$; $\tilde{a}(s):= a(s)$ if $s > t$.
	Such a modification of controls at a single point does not affect the trajectory or the total cost. Therefore, we have
	\begin{equation}
	\begin{split}
	v(\bx,t)\le Q(\bx,\bm a,t) &\leq \int_t^Tr({x}(s), \tilde{a}(s))\,\d s+q({x}(T))<v(\bx,t)+\e.
	\end{split}
	\end{equation}
	Since $\e$ was arbitrarily chosen, we conclude that $v(\bx,t)=Q(\bx,\bm a,t)$ for any $\bu\in \bbr^m$. \end{proof}

\subsection{Theorem~\ref{optimal}} \label{app:optimal}

\begin{proof}
	The classical theorem for the necessary and sufficient condition of optimality (e.g. Theorem 2.54, Chapter III in \cite{Bardi1997}) implies that $a^\star$ is optimal among those in $\mathcal{A}$ such that $a(t)={\bm a}$ if and only if
	\begin{align*}
	&p_1\cdot f(x^\star(s),a^\star(s))+p_2\cdot \dot{a}^\star(s) + r(x^\star(s),a^\star(s))\\
	& = \max_{|\bm{b}|\le L}\{p_1\cdot f(x^\star(s),a^\star(s))+p_2\cdot {\bm b} +r(x^\star(s),a^\star(s))\}
	\end{align*}
	for all $p=(p_1,p_2)\in D^{\pm}Q(x^\star(s),a^\star(s))$. This optimality condition can be expressed as the desired ODE \eqref{opt_con}. Thus, its solution $a^\star$ with $a^\star (t) = \bm{a}$ satisfies \eqref{optimality}. 
	
	Suppose now that ${\bm a}\in \argmax_{\bm{a}' \in \mathbb{R}^m} Q({\bm x},{\bm a}')$. It follows from the definition of $Q$ that 
	\begin{align*}
	&\max_{a\in \mathcal{A}}\left\{\int_t^\infty e^{-\gamma(s-t)} r(x(s),a(s))\,\d s\mid x(t)={\bm x}\right\}\\
	&=\max_{{\bm a}'\in\bbr^m}\max_{a\in \mathcal{A}}\bigg\{\int_t^\infty e^{-\gamma(s-t)} r(x(s),a(s))\,\d s \mid x(t)={\bm x},a(t)={\bm a}'\bigg \}\\
	&=\max_{{\bm a}'\in\bbr^m}Q({\bm x},{\bm a}')=Q({\bm x},{\bm a})=\max_{a\in \mathcal{A}}\left\{\int_t^\infty e^{-\gamma(s-t)} r(x(s),a(s))\d s \mid x(t)={\bm x},a(t)={\bm a}\right\}\\
	&=\int_t^\infty e^{-\gamma(s-t)} r(x^\star(s),a^\star(s))\,\d s.
	\end{align*}
	Therefore, $a^\star$ is an optimal control.
\end{proof}

\subsection{Proposition~\ref{prop:semi}} \label{app:semi}

\begin{proof}
	We first show that $Q^{h,\star}$ satisfies \eqref{semiHJB}. Fix an arbitrary sequence $b :=\{b_n\}_{n=0}^\infty\in\mathcal{B}$. It follows from the definition of $Q^{h,b}$ that
	\[
	Q^{h,b}(\bx,{\bm a})=hr(\bx,{\bm a})+(1-\gamma h) Q^{h,\tilde{b}}(\xi(\bx,{\bm a}; h), {\bm a}+hb_0).
	\]
	where $\tilde{b} := \{b_1,b_2,\ldots\}\in\mathcal{B}$. Since $Q^{h,\tilde{b}}( \xi(\bx,{\bm a}; h), {\bm a}+hb_0)\le Q^{h,\star}(\xi(\bx,{\bm a}; h), {\bm a}+hb_0)$, we have
	\begin{align*}
	Q^{h,b}(\bx,{\bm a})&\le hr(\bx,{\bm a})+(1-\gamma h)Q^{h,\star}(\xi(\bx,{\bm a}; h),{\bm a}+hb_0)\\
	&\le hr(\bx,{\bm a})+(1-\gamma h)\sup_{|{\bm b}|\le L} \left\{Q^{h,\star}(\xi(\bx,{\bm a}; h),{\bm a}+h{\bm b})\right\}.
	\end{align*}
	Taking supremum of both sides with respect to $b \in \mathcal{B}$ yields
	\begin{align}
	\begin{aligned}\label{ineq1}
	Q^{h,\star}(\bx,{\bm a}) \le hr(\bx,{\bm a})+(1-\gamma h)\sup_{|{\bm b}|\le L} \left\{Q^{h,\star}(\xi(\bx,{\bm a}; h),{\bm a}+h{\bm b})\right\}.
	\end{aligned}
	\end{align}
	
	To obtain the other direction of inequality, we fix an arbitrary ${\bm b}\in \bbr^m$ such that $|{\bm b}|\le L$.  Let $\bx' :=\xi(\bx,{\bm a}; h)$ and ${\bm a}' :={\bm a}+h{\bm b}$. Fix an arbitrary $\e>0$ and choose a sequence $c :=\{c_n\}_{n=0}^\infty\in\mathcal{B}$ such that
	\[
	Q^{h,\star}(\bx',{\bm a}')\le Q^{h,c}(\bx',{\bm a}')+\e.
	\]
	We now construct a new sequence $\tilde{c}:=\{{\bm b},c_0,c_1,\ldots\}\in\mathcal{B}$. Then,
	\begin{align*}
	Q^{h,\tilde{c}}(\bx,{\bm a})=hr(\bx,{\bm a})+(1-\gamma h) Q^{h,c}(\bx',{\bm a}')\ge hr(\bx,{\bm a})+(1-\gamma h)(Q^{h,\star}(\bx',{\bm a}')-\e), 
	\end{align*}
	which implies that
	\begin{align*}
	Q^{h,\star}(\bx,{\bm a})\ge Q^{h,\tilde{c}}(\bx,{\bm a})\ge hr(\bx,{\bm a})+(1-\gamma h)(Q^{h,\star}(\bx',{\bm a}')-\e).
	\end{align*}
	Taking the supremum of both sides with respect to ${\bm b} \in \mathbb{R}^m$ such that $| \bm{b} |\leq L$ yields
	\begin{align*}
	Q^{h,\star}(\bx,{\bm a})\ge hr(\bx,{\bm a})  -(1-\gamma h)\e+(1-\gamma h)\sup_{|{\bm b}|\le L}\left\{Q^{h,\star}(\xi(\bx,{\bm a}; h),{\bm a}+h{\bm b})\right\}.
	\end{align*}
	Since $\e$ was arbitrarily chosen, we finally obtain that
	\begin{align}
	\begin{aligned}\label{ineq2}
	Q^{h,\star}(\bx,{\bm a})\ge hr(\bx,{\bm a})+(1-\gamma h)\sup_{|{\bm b}|\le L}\left\{Q^{h,\star}(\xi(\bx,{\bm a}; h), {\bm a}+h{\bm b})\right\}.
	\end{aligned}
	\end{align}
	Combining two estimates \eqref{ineq1} and \eqref{ineq2}, we conclude that $Q^{h,\star}$ satisfies the semi-discrete HJB equation~\eqref{semiHJB}. Since the proof for the uniqueness of the solution is almost the same as the proof of Theorem 4.2, Chapter VI in \cite{Bardi1997}, we have omitted the detailed proof.
\end{proof}

\subsection{Proposition~\ref{conv1}} \label{app:conv1}

\begin{proof}
	For the completeness of the paper, we provide a sketch of the proof although it is similar to the proof of Theorem 1.1, Chapter VI in \cite{Bardi1997}. We begin by  defining two functions $\underline{Q}^\star$ and $\overline{Q}^\star$ as 
	\[\underline{Q}^\star(\bx,{\bm a}):= \liminf_{(\bx',{\bm a}',h)\to (\bx,{\bm a},0+)}Q^{h,\star}(\bx',{\bm a}'),\]
	\[\overline{Q}^\star(\bx,{\bm a}):= \limsup_{(\bx',{\bm a}',h)\to (\bx,{\bm a},0+)}Q^{h,\star}(\bx',{\bm a}').\]
	According to the proof of Theorem 1.1, Chapter VI in  \cite{Bardi1997}, it suffices to show that $\overline{Q}^\star$ satisfies the first condition of Definition~\ref{def:vis} and $\underline{Q}^\star$ satisfies the second condition of Definition~\ref{def:vis}. To this end, for any $\phi\in C^1$, let $(\bx_0,{\bm a}_0)$ be a strict local maximum point of $\overline{Q}^\star-\phi$ and choose a small enough neighborhood $\mathcal{N}$ of $(\bx_0,\bm{a}_0)$ such that $(\overline{Q}^\star-\phi)(\bx_0,{\bm a}_0)=\max_{\mathcal{N}}(\overline{Q}^\star-\phi)$. Then, there exists a sequence $\{(\bx_n,{\bm a}_n, h_n)   \}$ with $(\bx_n,{\bm a}_n)\to(\bx_0, {\bm a}_0)$ and $h_n\to 0+$ such that
	\[(Q^{h_n,\star}-\phi)(\bx_n,{\bm a}_n)=\max_{\mathcal{N}}(Q^{h_n,\star}-\phi)\]
	and
	\[Q^{h_n,\star}(\bx_n,{\bm a}_n)\to \overline{Q}^\star(\bx_0,{\bm a}_0).\]
	Recall that $Q^{h,\star}$ satisfies \eqref{semiHJB}. Thus, there exists ${\bm b}_n$ with $|{\bm b}_n|\le L$ such that
	\begin{align*}
	&Q^{h_n,\star}(\bx_n,{\bm a}_n)-h_nr(\bx_n,{\bm a}_n)-(1-\gamma h_n)Q^{h_n,\star}(\xi(\bx_n,{\bm a}_n;h_n),{\bm a}_n+h{\bm b}_n)=0.
	\end{align*}
	Since $Q^{h_n,\star}-\phi$ attains a local maximum at $(\bx_n,{\bm a}_n)$, we have
	\begin{align}
	\begin{aligned}\label{eq-prop3}
	(1-\gamma h_n)(\phi(\bx_n,{\bm a}_n)-\phi(\xi(\bx_n,{\bm a}_n;h_n),{\bm a}_n+h{\bm b}_n)+\gamma h_n Q^{h_n,\star}(\bx_n,{\bm a}_n)-h_nr(\bx_n,{\bm a}_n)\le0
	\end{aligned}
	\end{align}
	for small enough $h_n > 0$. Since $|{\bm b}_n|\le L$ for all $n\ge0$, there exists a subsequence $n_k$ and ${\bm b}$ with $|{\bm b}|\le L$ such that ${\bm b}_{n_k}\to {\bm b}$ as $k\to \infty$. Then, we substitute $n$ in \eqref{eq-prop3} by $n_k$, divide both sides by $h_{n_k}$ and let $k\to\infty$ to obtain that at $(\bx_0,{\bm a}_0)$
	\[-\nabla_{\bm x}\phi\cdot f-\nabla_{\bm a}\phi\cdot {\bm b}+\gamma \overline{Q}^{\star}-r\le0, \]
	where  we use the fact that
	\[\lim_{h\to 0} \frac{\xi(\bx,{\bm a};h)-\bx}{h}=f(\bx,{\bm a}).\]
	This implies that the first condition of Definition \ref{def:vis} is satisfied. 
	Similarly, it can be shown that $\underline{Q}^\star$ satisfies the second condition of Definition~\ref{def:vis}.
\end{proof}

\subsection{Theorem~\ref{thm:conv2}} \label{app:conv2}

We begin by defining an optimal Bellman operator in the semi-discrete setting, $\mathcal{T}^h: L^\infty \to L^\infty$,  by
\begin{equation}\label{bellman}
(\mathcal{T}^h Q)(\bm{x}, \bm{a}):= h r(\bm{x}, \bm{a})+ (1-\gamma h) \sup_{|\bm{b} | \leq L} Q  ( \xi(\bm{x}, \bm{a}; h),  \bm{a} + h \bm{b}),
\end{equation}
where $\xi(\bm{x}, \bm{a}; h)$ denotes the solution of the ODE \eqref{sys} at time $t = h$ with initial state $x(0) = \bm{x}$ and constant action $a(t) \equiv \bm{a}$ for $t \in [0, h)$.
Our first observation is that the Bellman operator is a monotone $(1-\gamma h)$-contraction mapping for a sufficiently small~$h$.

\begin{lemma}\label{lem:cont}
	Suppose that $0 < h < \frac{1}{\gamma}$.
	Then, the Bellman operator $\mathcal{T}^h$ is a monotone contraction mapping.
	More precisely, it satisfies the following properties: 
	\begin{enumerate}
		\item[(i)] $\mathcal{T}^h Q \leq \mathcal{T}^h Q'$ for all $Q, Q' \in L^\infty$ such that $Q  \leq Q'$;
		
		\item[(ii)] $\| \mathcal{T}^h Q - \mathcal{T}^h Q' \|_{L^\infty} \leq (1-\gamma h) \| Q - Q' \|_{L^\infty}$ for all $Q, Q' \in L^\infty$. 
	\end{enumerate}
\end{lemma}
\begin{proof}
	$(i)$ Since $Q(\bx,{\bm a})\le Q'(\bx,{\bm a})$ for all $(\bx,{\bm a}) \in \mathbb{R}^n \times \mathbb{R}^m$, we have
	\begin{align*}
	&\sup_{|{\bm b}|\le L} Q(\xi(\bx,{\bm a}; h),{\bm a}+h{\bm b})\le\sup_{|{\bm b}|\le L} Q'(\xi(\bx,{\bm a}; h),{\bm a}+h{\bm b}).
	\end{align*}
	Multiplying $(1-\gamma h)$ and then adding $hr(\bx,{\bm a})$ to both sides, we confirm the monotonicity of $\mathcal{T}^h$ as desired.
	
	$(ii)$ We first note that for any ${\bm b}\in \bbr^m$ with $|{\bm b}|\le L$,
	\begin{align*}
	&\big [ hr(\bx,{\bm a})+(1-\gamma h)Q(\xi(\bx,{\bm a}; h),{\bm a}+h{\bm b}) \big ] - \big [ hr(\bx,{\bm a})+(1-\gamma h)Q'(\xi(\bx,{\bm a}; h),{\bm a}+h{\bm b}) \big ]\\
	&=(1-\gamma h)\big [ Q(\xi(\bx,{\bm a}; h),{\bm a}+h{\bm b}) -Q'(\xi(\bx,{\bm a}; h),{\bm a}+h{\bm b}) \big ]\\
	& \le (1-\gamma h)\|Q-Q'\|_{L^\infty}.
	\end{align*}
	By the definition of $\mathcal{T}^hQ'$, we have
	\begin{align*}
	&hr(\bx,{\bm a})+(1-\gamma h)Q(\xi(\bx,{\bm a}; h),{\bm a}+h{\bm b})\\
	&\le (1-\gamma h)\|Q-Q'\|_{L^\infty} +hr(\bx,{\bm a})  +(1-\gamma h)Q'(\xi(\bx,{\bm a}; h),{\bm a}+h{\bm b})\\
	&\le(1-\gamma h)\|Q-Q'\|_{L^\infty} +\mathcal{T}^hQ'(\bx,{\bm a}).
	\end{align*}
	Taking the supremum of both sides with respect to $\bm{b} \in \mathbb{R}^m$ such that $|\bm{b}| \le L$, yields
	\[
	\mathcal{T}^hQ(\bx,{\bm a})\le (1-\gamma h)\|Q-Q'\|_{L^\infty} + \mathcal{T}^hQ'(\bx,{\bm a}),
	\]
	or equivalently
	\[
	\mathcal{T}^hQ(\bx,{\bm a})-\mathcal{T}^hQ'(\bx,{\bm a})\le (1-\gamma h)\|Q-Q'\|_{L^\infty}.
	\]
	We now change the role of $Q$ and $Q'$ to obtain
	\[
	|\mathcal{T}^hQ(\bx,{\bm a})-\mathcal{T}^hQ'(\bx,{\bm a})|\le (1-\gamma h)\|Q-Q'\|_{L^\infty}.
	\]
	Therefore, the operator $\mathcal{T}^h$ is a $(1-\gamma h)$-contraction with respect to $\| \cdot \|_{L^\infty}$.
\end{proof}

Using the Bellman operator $\mathcal{T}^h$, HJ Q-learning~\eqref{Q_sync} can be expressed as
\[
{Q}_{k+1}^h := (1-\alpha_k) {Q}_k^h + \alpha_k \mathcal{T}^h {Q}_k^h.
\]
Consider the difference $\Delta_k^h :=  {Q}_k^h -  Q^{h,\star}$. Note that $\| \Delta_k^h\|_{L^\infty}$ represents the optimality gap at the $k$th iteration. 
It satisfies 
\begin{equation}\label{R_equation}
\Delta_{k+1}^h = (1-\alpha_k) \Delta_k^h + \alpha_k [ \mathcal{T}^h (\Delta_k^h + Q^{h,\star}) - \mathcal{T}^h Q^{h,\star} ],
\end{equation}
where we used the semi-discrete HJB equation $Q^{h,\star} = \mathcal{T}^h Q^{h, \star}$. 
The contraction property of the Bellman operator $\mathcal{T}^h$ 
can be used to show that the optimality gap $\| \Delta_k^h\|_{L^\infty}$ decreases geometrically. 
More precisely, we have the following lemma:

\begin{lemma}\label{lem:geo}
	Suppose that $0 < h < \frac{1}{\gamma}$, $0\le \alpha_k\le 1$ and that Assumption~\ref{ass:bl} holds.
	Then, the following inequality holds:
	\[
	\| \Delta_k^h \|_{L^\infty} \leq \bigg (
	\prod_{\tau=0}^{k-1} (1 - \alpha_\tau \gamma h )
	\bigg )  \| \Delta_0^h \|_{L^\infty}.
	\]
\end{lemma}
\begin{proof}
	We use mathematical induction to prove the assertion. When $k=1$, it follows from the Q-function update~\eqref{Q_sync} and the contraction property of $\mathcal{T}^h$ that
	\begin{align*}
	\|\Delta^h_1\|_{L^\infty}&\le (1-\alpha_0)\|\Delta^h_0\|_{L^\infty} +\alpha_0\|\mathcal{T}^h(\Delta^h_0+Q^{h,\star})-\mathcal{T}^hQ^{h,\star}\|_{L^\infty}\\
	&\le (1-\alpha_0)\|\Delta_0^h\|_{L^\infty} +\alpha_0(1-\gamma h)\|\Delta^h_0\|_{L^\infty}\\
	&=(1-\alpha_0\gamma h)\|\Delta^h_0\|_{L^\infty}.
	\end{align*}
	Therefore, the assertion holds for $k=1$. We now assume that the assertion holds for $k=n$:
	\[
	\|\Delta^h_n\|_{L^\infty}\le \left(\prod_{\tau=0}^{n-1}(1-\alpha_\tau \gamma h)\right) \|\Delta^h_0\|_{L^\infty}.
	\]
	We need to show that the inequality holds for $k=n+1$. By using the same estimate as in the case of $k=1$ and the induction hypothesis for $k=n$, we obtain
	\begin{align*}
	\|\Delta^h_{n+1}\|_{L^\infty}&\le (1-\alpha_n)\|\Delta^h_n\|_{L^\infty}  +\alpha_n\|\mathcal{T}^h(\Delta^h_n+Q^{h,\star})-\mathcal{T}^hQ^{h,\star}\|_{L^\infty}\\
	&\le (1-\alpha_n)\|\Delta^h_n\|_{L^\infty} +\alpha_n(1-\gamma h)\|\Delta^h_n\|_{L^\infty}\\
	&=(1-\alpha_n\gamma h)\|\Delta^h_n\|_{L^\infty}\\
	&\le (1-\alpha_n\gamma h)\left(\prod_{\tau=0}^{n-1}(1-\alpha_\tau\gamma h)\right) \|\Delta^h_0\|_{L^\infty}\\
	&=\left(\prod_{\tau=0}^{n}(1-\alpha_\tau\gamma h)\right) \|\Delta^h_0\|_{L^\infty}.
	\end{align*}
	This completes our mathematical induction, and thus the result follows. 
\end{proof}

This lemma yields a condition on the sequence of learning rates under which the Q-function updated by \eqref{Q_sync} converges to the optimal semi-discrete Q-function~\eqref{Q_semi} in $L^\infty$.

\begin{proof}
	It suffices to show that 
	\[
	\lim_{k\to \infty}\|\Delta_k^h\|_{L^\infty}=0.
	\]
	By Lemma~\ref{lem:geo} and the elementary inequality $1-x\le e^{-x}$, we have
	\begin{align*}
	\|\Delta_{k}^h\|_{L^\infty} \le \left(\prod_{\tau=0}^{k-1}(1-\alpha_\tau \gamma h)\right)\|\Delta_0^h\|_{L^\infty} \le \exp\left(-\gamma h\left(\sum_{\tau=0}^{k-1} \alpha_\tau\right)\right)\|\Delta_0^h\|_{L^\infty}.
	\end{align*}
	Therefore, if $\sum_{\tau=0}^\infty \alpha_\tau=\infty$, the result follows.
\end{proof}

\subsection{Corollary~\ref{cor:conv}} \label{app:cor:conv}

\begin{proof}
	We first observe that there exists an index $k_h$, depending on $h$, such that $\sum_{\tau=0}^{k_h-1} \alpha_\tau>\frac{1}{h^2}$ since $\sum_{\tau=0}^\infty \alpha_\tau=\infty$. Then, we have
	\[
	h\left(\sum_{\tau=0}^{k_h-1}\alpha_\tau\right) >\frac{1}{h}\to \infty\quad\mbox{as}\quad h\to0.
	\]
	Moreover, by the triangle inequality, we have
	\begin{align*}
	|Q^h_k(\bx,{\bm a})-{{Q}}(\bx,{\bm a})|\le |Q^h_k(\bx,{\bm a})-Q^{h,\star}(\bx,{\bm a})| +|Q^{h,\star}(\bx,{\bm a})-{{Q}}(\bx,{\bm a})|
	\end{align*}
	for all $(\bm{x}, \bm{a}) \in \mathbb{R}^n \times \mathbb{R}^m$.
	By Proposition 2, the second term on the right-hand side uniformly vanishes over any compact subset $K$ of $\bbr^n\times\bbr^m$ as $h\to 0$. The first term is nothing but $|\Delta_k^h(\bx,{\bm a})|$, which is bounded as follows (by Lemma~\ref{lem:geo}):
	\begin{align*}
	|\Delta_k^h(\bx,{\bm a})|&\le \left(\prod_{\tau=0}^{k-1}(1-\alpha_\tau\gamma h)\right)\|\Delta_0^h\|_{L^\infty}\le \exp\left(-\gamma h\left(\sum_{\tau=0}^{k-1} \alpha_\tau\right)\right)\|\Delta_0^h\|_{L^\infty},\quad k\ge 1,
	\end{align*}
	where the second inequality holds because $1-x\le e^{-x}$. Our choice of $k_h$ then yields
	\[
	\sup_{k\ge k_h} \|\Delta_k^h\|_{L^\infty}\le\exp\left(-\gamma h\left(\sum_{\tau=0}^{k_h-1} \alpha_\tau\right)\right)\|\Delta_0^h\|_{L^\infty}\to 0
	\]
	as $h\to 0$.
	Therefore, we conclude that
	\begin{align*}
	&\sup_{k\ge k_h} \sup_{\substack{(\bm{x}, \bm{a}) \in K\\ K \mathrm{\tiny compact}}}  |Q_k^h(\bx,{\bm a})-{{Q}}(\bx,{\bm a})| \\
	&\le \sup_{k\ge k_h} \sup_{\substack{(\bm{x}, \bm{a}) \in K\\ K \mathrm{\tiny compact}}} |Q_k^h(\bx,{\bm a})-Q^{h,\star}(\bx,{\bm a})| +\sup_{k\ge k_h} \sup_{\substack{(\bm{x}, \bm{a}) \in K\\ K \mathrm{\tiny compact}}} |Q^{h,\star}(\bx,{\bm a})-{{Q}}(\bx,{\bm a})|\\
	&\le \sup_{k\ge k_h}\|\Delta_k^h\|_{L^\infty}  +\sup_{\substack{(\bm{x}, \bm{a}) \in K\\ K \mathrm{\tiny compact}}} |Q^{h,\star}(\bx,{\bm a})-{{Q}}(\bx,{\bm a})|\to 0
	\end{align*}
	as $h\to0$.
\end{proof}

\subsection{Proposition~\ref{prop:diff}} \label{app:prop:diff}

\begin{proof}
	We first notice  that by the triangle inequality,
	\begin{align*}
	&\left|\max_{|\bm{a} - a_j| \leq hL}Q_{\theta^-}(x_{j+1},{\bm a})-Q_{\theta^-}(x_{j+1},a_j+hb_j)\right|\\
	&\le \bigg |\max_{|\bm{a} - a_j| \leq hL}Q_{\theta^-}(x_{j+1},{\bm a}) -Q_{\theta^-}\bigg(x_{j+1},a_j+hL\frac{\nabla_{\bm a}Q_{\theta^-}(x_{j+1},a_j)}{|\nabla_{\bm a}Q_{\theta^-}(x_{j+1},a_j)|}\bigg)\bigg |\\
	&\quad+\bigg|Q_{\theta^-}\bigg(x_{j+1},a_j+hL\frac{\nabla_{\bm a}Q_{\theta^-}(x_{j+1},a_j)}{|\nabla_{\bm a}Q_{\theta^-}(x_{j+1},a_j)|}\bigg)-Q_{\theta^-}(x_{j+1},a_j+hb_j)\bigg| \\
	&=: \Delta_1+\Delta_2.
	\end{align*}
	
	We first consider $\Delta_1$. Let ${\bm a}^\star := \argmax_{|{\bm a}-a_j|\le hL}$ $Q_{\theta^-}(x_{j+1},{\bm a})$. By the Taylor expansion, we have
	\begin{align*}
	\max_{|\bm{a} - a_j| \leq hL}Q_{\theta^-}(x_{j+1},{\bm a}) &= Q_{\theta^-}(x_{j+1},{\bm a}^\star)\\
	& =Q_{\theta^-}(x_{j+1},a_j)+\nabla_{\bm a} Q_{\theta^-}(x_{j+1},a_j)\cdot({\bm a}^\star-a_j)+O(h^2).
	\end{align*}
	Similarly, we again use the Taylor expansion to obtain that
	\begin{align*}
	Q_{\theta^-}\left(x_{j+1},a_j + hL \frac{\nabla_{\bm a}Q_{\theta^-}(x_{j+1},a_j)}{|\nabla_{\bm a}Q_{\theta^-}(x_{j+1},a_j)|}\right) =Q_{\theta^-}(x_{j+1},a_j)+hL |\nabla_{\bm a}Q_{\theta^-}(x_{j+1},a_j)|+O(h^2).
	\end{align*}
	Subtracting one equality from another yields
	\begin{align*}
	&Q_{\theta^-}(x_{j+1},{\bm a}^\star)-Q_{\theta^-}\left(x_{j+1},a_j + hL \frac{\nabla_{\bm a}Q_{\theta^-}(x_{j+1},a_j)}{|\nabla_{\bm a}Q_{\theta^-}(x_{j+1},a_j)|}\right)\\
	&=\nabla_{\bm a} Q_{\theta^-}(x_{j+1},a_j)\cdot({\bm a}^\star-a_j) -hL |\nabla_{\bm a}Q_{\theta^-}(x_{j+1},a_j)|  +O(h^2) \le O(h^2),
	\end{align*}
	where the last inequality holds because $|\bm{a}^\star -a_j | \leq h L$. Since our choice of ${\bm a}^\star$ implies that the left-hand side of the inequality above is always non-negative, we conclude that $\Delta_1=O(h^2)$.
	
	Regarding $\Delta_2$, we have
	\begin{align*}
	\Delta_2&=\bigg |Q_{\theta^-}\bigg(x_{j+1},a_j + hL \frac{\nabla_{\bm a}Q_{\theta^-}(x_{j+1},a_j)}{|\nabla_{\bm a}Q_{\theta^-}(x_{j+1},a_j)|}\bigg) -Q_{\theta^-}\bigg(x_{j+1},a_j + hL \frac{\nabla_{\bm a}Q_{\theta^-}(x_{j},a_j)}{|\nabla_{\bm a}Q_{\theta^-}(x_{j},a_j)|}\bigg)\bigg|\\
	&\le Lh\|\nabla_{\bm a}Q_{\theta^-}\|_{L^\infty} \bigg |\frac{\nabla_{\bm a}Q_{\theta^-}(x_{j+1},a_j)}{|\nabla_{\bm a}Q_{\theta^-}(x_{j+1},a_j)|}-\frac{\nabla_{\bm a}Q_{\theta^-}(x_{j},a_j)}{|\nabla_{\bm a}Q_{\theta^-}(x_{j},a_j)|}\bigg |.
	\end{align*}
	Note that for any two non-zero vectors $v,w$, 
	\begin{align*}
	\left|\frac{v}{|v|}-\frac{w}{|w|}\right|\le \frac{|v-w|}{|v|} + |w|\left(\frac{1}{|v|}-\frac{1}{|w|}\right)= \frac{|v-w|}{|v|}+\frac{|w|-|v|}{|v|}\le \frac{2|v-w|}{|v|}.
	\end{align*}
	On the other hand, we have 
	\begin{align*}
	|\nabla_{\bm a}Q_{\theta^-}(x_{j+1},a_j)-\nabla_{\bm a}Q_{\theta^-}(x_j,a_j)|\le \|\nabla_{\bx \bm a }^2 Q_{\theta^-}\|_{L^\infty}|x_{j+1}-x_j|=O(h).
	\end{align*}
	Since we assume that $Q_{\theta^-}$ is twice differentiable and $|\nabla_{\bm a} Q_{\theta^-}(x_j,a_j)|=:C>0$, we have $|\nabla_{\bm a}Q_{\theta^-}(x_{j+1},a_j)|> {C}/{2}$ for sufficiently small $h$. Therefore, we obtain that
	\begin{align*}
	\Delta_2 &\le 2Lh\|\nabla_{\bm a}Q_{\theta^-}\|_{L^\infty} \frac{|\nabla_{\bm a}Q_{\theta^-}(x_{j+1},a_j)-\nabla_{\bm a}Q_{\theta^-}(x_j,a_j)|}{|\nabla_{\bm a}Q_{\theta^-}(x_{j+1},a_j)|}=O(h^2).
	\end{align*}
	Combining the estimates of $\Delta_1$ and $\Delta_2$ yields	
	\begin{equation} \nonumber
	\begin{split}
	&\left|\max_{|\bm{a} - a_j| \leq hL} Q_{\theta^-} (x_{j+1}, \bm{a})  - Q_{\theta^-} (x_{j+1}, a_j + h b_j)\right|=O(h^2)
	\end{split}
	\end{equation}
	as desired.
\end{proof}

\section{Brief Discussion on Extension to Stochastic Systems}	\label{app:stoch}

The Hamilton-Jacobi Q-learning can be extended to the continuous-time stochastic control setting with controlled diffusion processes. 
Consider the following stochastic counterpart of the system \eqref{sys}:
\begin{equation}\label{sys_stoch}
\d x_t = f(x_t,a_t)\d t+\sigma(x_t,a_t)\d W_t,\quad t>0,
\end{equation}
where $\sigma:\bbr^n\times\bbr^m\to \bbr^{n\times k}$ is the diffusion coefficient and $W_t$ is the $k$-dimensional standard Bronwian motion. We now define the Q-function as
\[Q(\bm{x},\bm{a}):=\sup_{a\in \mathcal{A}}\mathbb{E}\left[\int_0^\infty e^{-\gamma t}r(x_t,a_t)\d t ~\mid~x_0=\bm{x},a_0=\bm{a}\right].\]
Again, the dynamic programming principle implies
\begin{equation}
\begin{split}
&0= \sup_{a \in \mathcal{A}}\mathbb{E} 
\bigg [
\frac{1}{h} \int_t^{t+h} e^{-\gamma (s-t)} r(x(s), a(s)) \: \mathrm{d} s+ \frac{1}{h} [Q(x(t+h), a(t+h)) - Q(\bm{x}, \bm{a}) ] \\
&\hspace{2cm} + \frac{e^{-\gamma h} - 1}{h} Q(x(t+h), a(t+h))\mid x(t) = \bm{x}, a(t) = \bm{a}
\bigg ].
\end{split}
\end{equation}

Then, we use the It\^o formula
\begin{align*}
\d Q(x_t,a_t) &= \nabla_{\bm x} Q\cdot \d x_t +\nabla_{\bm a}Q \cdot {\dot a} \d t + \frac{1}{2}\d x_t^\top \nabla_{\bm x}^2 Q \d x_t\\
&=\nabla_{\bm x} Q\cdot (f(x_t,a_t)\d t +\sigma(x_t,a_t)\d W_t) +\nabla_{\bm a}Q\cdot \dot{a} \d t+\frac{(\d W_t)^\top\sigma^\top\nabla_{\bm x}^2 Q\sigma\d W_t}{2}
\end{align*}
to derive the following Hamilton-Jacobi-Bellman equation for the stochastic  system \eqref{sys_stoch}:
\begin{equation}\label{HJB-stoch}
\gamma Q -\nabla_{\bm{x}}Q\cdot f(\bm{x},\bm{a})-L|\nabla_{\bm{a}} Q| -r(\bm{x},\bm{a}) -\frac{\textup{tr}(\sigma^\top\nabla_{\bx}^2 Q\sigma)}{2}=0.
\end{equation}
Note that, in this case also, the optimal control satisfies  $\dot{a} = L\frac{\nabla_{\bm a}Q}{|\nabla_{\bm a}Q|}$ when $Q$ is differentiable.

Since in most practical systems transition samples are collected in discrete time, we also introduce the semi-discrete version of \eqref{HJB-stoch}. We define a stochastic semi-discrete Q-function $Q^{h,\star}$ as 
\[Q^{h,\star}(\bx,\bm{a}) := \sup_{b\in\mathcal{B}} \mathbb{E}\left[h\sum_{k=0}^\infty r(x_k,a_k)(1-\gamma h)^k\right],\]
where $\mathcal{B}:=\{b:=\{b_k\}_{k=0}^\infty~\mid~ b_k\in\bbr^m,|b_k|\le L\}$, $x_{k+1}=\xi(x_k,a_k;h)$ and $a_{k+1} = a_k+hb_k$. Here, $\xi(x_k,a_k;h)$ is now a solution to the stochastic differential equation \eqref{sys_stoch} at time $t=h$ with initial state $\bx$ and constant control $a(t)\equiv {\bm a}$, $t\in[0,h)$. Then, similar to the deterministic semi-discrete HJB equation \eqref{semiHJB}, its stochastic counterpart can be written as follows:
\begin{align*}
Q^{h,\star}(\bm{x},\bm{a}) = hr(\bm{x},\bm{a})+(1-\gamma h)\sup_{|{\bm b}|\le L}\mathbb{E}\left[Q^{h,\star}(\xi(\bm{x},\bm{a};h),{\bm a}+h{\bm b})\right].
\end{align*}
Using Robbins-Monro stochastic approximation \cite{Robbins1951, Kushner2003}, we obtain the following model-free update rule: in the $k$th iteration, 
we collect data $(x_k, a_k, r_k, x_{k+1})$ and update the Q-function by
\begin{align}
\begin{aligned}\label{Q_stoch}
{Q}_{k+1}^h (x_k, a_k) := (1-\alpha_k) {Q}_k^h (x_k, a_k) + \alpha_k \Big [ h r_k  + (1-\gamma h) \sup_{|\bm{b} | \leq L} {Q}_k^h (x_{k+1},  a_k + h \bm{b}) 
\Big ],
\end{aligned}
\end{align}
where $x_{k+1}$ is obtained by simulating the stochastic system from $x_k$ with action $a_k$ fixed for $h$ period, i.e., $x_{k+1} = \xi(x_k, a_k; h)$.  
The corresponding HJ DQN algorithm for stochastic systems is essentially the same as Algorithm~\ref{alg:HJ} although the transition samples are now collected through the stochastic system.

\begin{table*}[t]
	\centering
	\caption{Hyperparameters for HJ DQN.}
	\vspace{0.1in}
	\begin{tabular}{l | *{3}{c}}
		\hline
		Hyperparameter  			  			& HalfCheetah-v2 & Hopper-v2 & Walker2d-v2 \\
		\hline \hline
		optimizer      							& \multicolumn{3}{c}{Adam~\cite{Kingma2015}}          \\
		learning rate  							& $5\times 10^{-4}$ & $10^{-4}$ & $10^{-4}$\\ 
		Lipschitz constant $(L)$ 				& 30 & 30 & 30    \\
		default sampling interval $(h)$    & 0.05  & 0.008  & 0.008 \\
				tuned sampling interval $(h)$    &  0.01 &  0.016  & 0.032 \\
		(Continuous) discount $(\gamma)$       	& \multicolumn{3}{c}{$-\log(0.99)/h$, where $h$ is the sampling interval}\\
		replay buffer size       				& \multicolumn{3}{c}{$10^6$}           \\
		target smoothing coefficient $(\alpha)$	& \multicolumn{3}{c}{0.001}   \\
		
		Noise coefficient $(\sigma)$			& \multicolumn{3}{c}{0.1}\\
		number of hidden layers 				& \multicolumn{3}{c}{2 (fully connected)}         \\
		number of hidden units per layer  		& \multicolumn{3}{c}{256}	\\
		number of samples per minibatch 		& \multicolumn{3}{c}{128}    \\
		nonlinearity 							& \multicolumn{3}{c}{ReLU}  \\
		\hline
	\end{tabular}
	\begin{tabular}{ *{2}{c}}
	\hline
	  Swimmer-v2 & LQ   \\
	\hline \hline
	\multicolumn{2}{c}{Adam~\cite{Kingma2015}}          \\
	 $5\times 10^{-4}$ & $10^{-3}$         \\ 
	  15 & 10   \\
	  0.04 &  0.05   \\
	  	0.08  & -   \\
	 $-\log(0.99)/h$ & $-\log(0.99999)/h$ \\
	 $10^6$ & $2\times 10^4$           \\
	 \multicolumn{2}{c}{0.001}   \\
	
	\multicolumn{2}{c}{0.1}\\
	\multicolumn{2}{c}{2 (fully connected)}         \\
	\multicolumn{2}{c}{256}	\\
	128 & 512    \\
	\multicolumn{2}{c}{ReLU}  \\
	\hline
\end{tabular}
	\label{hyperparameter}
\end{table*}

\begin{table*}[t]
	\centering
	\caption{Hyperparameters for DDPG.}
	\vspace{0.1in}
	\begin{tabular}{l|*{2}{c} }
		\hline
		Hyperparameter  			  			& MuJoCo tasks & LQ \\
		\hline \hline
		optimizer      							&\multicolumn{2}{c}{Adam~\cite{Kingma2015}}          \\
		actor learning rate  					& \multicolumn{2}{c}{$10^{-4}$}          \\
		critic learning rate 					& \multicolumn{2}{c}{$10^{-3}$}	\\ 
		(Discrete) discount $(\gamma')$       	& 0.99  & 0.99999 \\
		replay buffer size       				& $10^6$ &$2\times 10^4$    \\
		target smoothing coefficient $(\alpha)$	& \multicolumn{2}{c}{0.001}  \\
		number of hidden layers 				& \multicolumn{2}{c}{2 (fully connected)}         \\
		number of hidden units per layer  		& \multicolumn{2}{c}{256}		\\
		number of samples per minibatch 		& 128 & 512	  \\
		nonlinearity 							& \multicolumn{2}{c}{ReLU} 	 \\
		\hline
	\end{tabular}
	\label{hyperparameter_other}
\end{table*}

\section{Implementation Details}\label{app:detail}

All the simulations in Section~\ref{sec:exp} were conducted using Python 3.7.4 on a PC with Intel Core i9-9900X @ 3.50GHz, NVIDIA GeForce RTX 2080 Ti and 64GB RAM.

Table \ref{hyperparameter} shows the list of hyperparameters that are used in our implementation of HJ DQN for each MuJoCo task and the LQ problem.
 For DDPG, we list our choice of hyperparameters in Table \ref{hyperparameter_other}, which are taken from \cite{Lillicrap2016} for MuJoCo tasks, except the network architecture which is used in OpenAI's implementation of DDPG\footnote{https://github.com/openai/spinningup}. The discount factor in the discrete-time algorithms is chosen as $\gamma'=0.99$ for MuJoCo tasks and $0.99999$ for the LQ problem so that it is equivalent to $e^{-\gamma h}\approx(1-\gamma h)$ in our algorithm for continuous-time systems.

\bibliographystyle{IEEEtran}

\bibliography{reference}

\end{document}